\documentclass[sigconf]{acmart}
\AtBeginDocument{%
  \providecommand\BibTeX{{%
    Bib\TeX}}}

\setcopyright{acmlicensed}
\copyrightyear{2018}
\acmYear{2018}
\acmDOI{XXXXXXX.XXXXXXX}
\acmConference[Conference acronym 'XX]{Make sure to enter the correct
  conference title from your rights confirmation email}{June 03--05,
  2018}{Woodstock, NY}
\acmISBN{978-1-4503-XXXX-X/2018/06}

\usepackage[scr=boondox,cal=esstix]{mathalpha}

\usepackage{amsmath,amssymb,amsfonts}
\usepackage{algorithmic}
\usepackage{amsthm}
\usepackage{graphicx}
\usepackage{booktabs}
\usepackage{multirow}
\usepackage{caption}
\usepackage{textcomp}
\usepackage{xcolor}
\usepackage{graphicx}
\usepackage{url}
\usepackage{color}
\usepackage{alltt}
\usepackage{bbding}
\usepackage{newfloat}
\usepackage{multicol}
\usepackage{tikz}
\usepackage{filecontents} 
\usepackage{epstopdf} 
\usepackage{xspace}
\usepackage{enumitem}
\usepackage{svg}
\usepackage[ruled,linesnumbered,noend]{algorithm2e}
\usepackage{adjustbox}
\usepackage{wrapfig}
\usepackage{graphbox}
\usepackage{rotating}
\usepackage{pdfpages}
\usepackage{mdwlist}
\usepackage{multirow}
\usepackage{pifont}
\usepackage{indentfirst}

\usepackage{subfig}
\usepackage{subcaption}
\usepackage[labelfont=bf,skip=0pt]{caption}
\usepackage{hyperref} 
\usepackage[capitalize,nameinlink]{cleveref}
\usepackage{listings}

\def\BibTeX{{\rm B\kern-.05em{\sc i\kern-.025em b}\kern-.08em
    T\kern-.1667em\lower.7ex\hbox{E}\kern-.125emX}}

\newcommand{\eat}[1]{}
\newcommand{\eg}{e.g.,\xspace}
\newcommand{\ie}{i.e.,\xspace}

\newtheorem{proposition}{Proposition}

\begin{document}

\title{Neutral Agent-based Adversarial Policy Learning against Deep Reinforcement Learning in Multi-party Open Systems}

\author{Qizhou Peng}
\affiliation{%
  \institution{State Key Laboratory of Cyberspace Security Defense, Institute of Information Engineering, Chinese Academy of Sciences}
  \city{Beijing}
  \country{China}}
\email{pengqizhou@iie.ac.cn}

\author{Yang Zheng}
\authornote{Corresponding author.}

\affiliation{%
  \institution{State Key Laboratory of Cyberspace Security Defense, Institute of Information Engineering, Chinese Academy of Sciences}
  \city{Beijing}
  \country{China}
}
\email{zhengyang@iie.ac.cn}

\author{Yu Wen}
\affiliation{%
 \institution{State Key Laboratory of Cyberspace Security Defense, Institute of Information Engineering, Chinese Academy of Sciences}
 \city{Beijing}
 \country{China}}
\email{wenyu@iie.ac.cn}

\author{Yanna Wu}
\affiliation{%
 \institution{State Key Laboratory of Cyberspace Security Defense, Institute of Information Engineering, Chinese Academy of Sciences}
 \city{Beijing}
 \country{China}}
\email{wuyanna@iie.ac.cn}

\author{Yingying Du}
\affiliation{%
 \institution{State Key Laboratory of Cyberspace Security Defense, Institute of Information Engineering, Chinese Academy of Sciences}
 \city{Beijing}
 \country{China}}
\email{duyingying@iie.ac.cn}

\begin{abstract}
Reinforcement learning (RL) has been an important machine learning paradigm for solving long-horizon
sequential decision-making problems under uncertainty. By integrating deep neural networks (DNNs) into the RL framework, deep reinforcement learning (DRL) has emerged, which achieved significant success in various domains. However, the integration of DNNs also makes it vulnerable to adversarial attacks. Existing adversarial attack techniques mainly focus on either directly manipulating the environment with which a victim agent interacts or deploying an adversarial agent that interacts with the victim agent to induce abnormal behaviors. While these techniques achieve promising results, their adoption in multi-party open systems remains limited due to two major reasons: impractical assumption of full control over the environment and dependent on interactions with victim agents.

To enable adversarial attacks in multi-party open systems, in this paper, we redesigned an adversarial policy learning approach that can mislead well-trained victim agents without requiring direct interactions with these agents or full control over their environments. Particularly, we propose a neutral agent-based approach across various task scenarios in multi-party open systems. While the neutral agents seemingly are detached from the victim agents, indirectly influence them through the shared environment. 
We evaluate our proposed method on the SMAC platform based on Starcraft II and the autonomous driving simulation platform Highway-env. The experimental results demonstrate that our method can launch general and effective adversarial attacks in multi-party open systems.

\end{abstract}



\received{20 February 2007}
\received[revised]{12 March 2009}
\received[accepted]{5 June 2009}

\maketitle

\section{Introduction}\label{section_intro}


Reinforcement learning (RL) is an important paradigm in machine learning for making a sequence of decisions under uncertainty. In this paradigm, a RL agent, as the decision-making entity, interacts with the environment through trial and error, learns to select actions based on observations and gradually adapts its policy that maximize cumulative rewards.
Recently, driven by advances in deep learning, deep reinforcement learning (DRL) has emerged by integrating deep neural networks (DNNs) into the RL framework. The integration has enabled DRL to achieve remarkable success across various domains, including strategic games  (e.g., AlphaGo~\cite{silver2016mastering}), robotics research ~\cite{andrychowicz2020learning, masmitja2023dynamic, radosavovic2024real}, autonomous driving~\cite{bojarski2016end, van2016coordinated, chen2023deep, chen2021automatic}, and training of large language models~\cite{guo2025deepseek}.

Despite these advances, DRL inherits the vulnerability of DNNs to \textit{adversarial attacks}~\cite{carlini2017towards, goodfellow2014explaining, papernot2016limitations}, where carefully crafted perturbations can cause models to produce incorrect outputs.
This vulnerability has raised increasing concern from the security community.
Existing studies\cite{behzadan2017vulnerability, huang2017adversarial, zan2023adversarial, liu2023efficient,gleave2019adversarial, wu2021adversarial, guo2021adversarial} have demonstrated that adversarial attacks can induce DRL agents to behave sub-optimally or dangerously by introducing subtle perturbations to their inputs, observations, or environment dynamics.
These adversarial attack techniques often fall into two categories: (1)
\textit{environment manipulation-based methods} focus on directly manipulating the environment with which a victim agent interacts to perturb its observations and mislead the agent to behave abnormally~\cite{behzadan2017vulnerability, huang2017adversarial, zan2023adversarial, liu2023efficient};
(2) \textit{adversarial policy learning-based methods} employ a self-deployed adversarial agent that interacts with a victim agent to observe its behaviors, infer its policy, and learn an adversarial policy that guides the adversarial agent to take actions aimed at misleading the victim agent~\cite{gleave2019adversarial, wu2021adversarial, guo2021adversarial}. 

While these techniques have achieved promising results, little attention has been devoted to \textit{multi-party open systems}, which represent a class of particularly complex DRL task environments in practice, such as autonomous driving environments ~\cite{Wei2021Autonomous, Miao2025CCMA}, and complex strategy games (\eg Starcraft and Civilization) ~\cite{vinyals2019grandmaster, Meta2022Human}. 
Unlike the DRL task environments mainly focused in existing work (\eg single-agent operations ~\cite{bojarski2016end,masmitja2023dynamic}, two-agent competitions ~\cite{silver2018general}, multi-agent cooperation ~\cite{van2016coordinated, chen2023deep}), which typically restrict agents to a fixed number or a small number of (\ie 1$\sim$2) parties, the agents in multi-party open systems can be freely deployed, and be organized into one or more parties, each comprising at least one agent. 
Specifically, existing techniques face two main limitations in multi-party open systems:
\textit{(1) Impractical assumption of full control over the environment.} Environment manipulation-based methods assume full control over the environment for adversarial manipulation, which proves impractical given the excessive time and computational resources required to hack into the environment or a victim agent;
\textit{(2) Dependent on interactions with victim agents.} Adversarial policy learning-based methods often require adversarial agents to have interactions (\eg competition interactions ~\cite{gleave2019adversarial, wu2021adversarial, guo2021adversarial, ma2024sub} and cooperation interactions ~\cite{li2023attacking}) with victim agents in the same tasks to mislead them. However, in open multi-party open systems, adversarial agents cannot always have the opportunity to participate in the victim agents' tasks to have such interactions. 

\noindent\textbf{Our solution.} In this paper, \textbf{we propose a neutral agent-based adversarial policy learning approach that misleads well-trained victim agents in multi-party open systems without requiring direct interactions with these agents and full control over the systems.} 
By carefully examining various multi-party open systems~\cite{Wei2021Autonomous, Miao2025CCMA, vinyals2019grandmaster, Meta2022Human}, we observe that agents in these systems can be deployed in neutral roles, which do not participate in any interactions with other agents. 
Moreover, we observe a neutral agent functions like a bystander that, while having no direct interactions with other agents, can observe these agents and subtly adjusts its actions to indirectly influence them through the shared environment.
Based on these observations, more specifically, we train neutral agents to learn adversarial policies (\ie repurpose neutral agents as adversarial agents) by first designing an appropriate reward to guide policy optimization, and then developing an efficient computation method to perform this optimization.
At a high level, our method extends adversarial policy learning to neutral agents that do not directly interact with victim agents. This method naturally inherits a key advantage of existing adversarial policy learning methods: it does not require full control over task environments.
To the best of our knowledge, our approach is the first to leverage neutral agent–based adversarial policy learning to attack DRL in multi-party open systems. 
Notably, although our method is designed for multi-party open systems, it is also applicable to other task environments, including single-agent operations~\cite{bojarski2016end, masmitja2023dynamic}, two-agent competitions~\cite{silver2018general}, and multi-agent cooperation~\cite{van2016coordinated, chen2023deep}, as they can be considered special cases of multi-party open systems.

 
\noindent\textbf{Challenges.} Although leveraging neutral agents for adversarial policy learning holds promise for attacking DRL in multi-party open systems,  the unique characteristics of such systems and their internal neutral agents pose two major challenges to its effective implementation.

\noindent\ding{182}
In adversary attacks against DRL, the ultimate objective of an adversary is to induce task failures in victim agents~\cite{gleave2019adversarial}, which can be realized by designing adversary rewards that incentivize actions leading to such failures.
Existing adversary reward designs\cite{gleave2019adversarial, wu2021adversarial, guo2021adversarial} typically focus on zero-sum RL tasks by simply taking the negative of the victim agents’ rewards. However, in multi-party open systems that often involve non-zero-sum RL tasks~\cite{Wei2021Autonomous, Miao2025CCMA, vinyals2019grandmaster, Meta2022Human}, the rewards of the victim agents are often private, making it challenging for neutral agents to access them.

To address this challenge, we design a novel adversary reward by leveraging the failure paths of victim agents. Fundamentally, an important paradigm of reward design is to formulate metrics that effectively measure the performance of RL tasks to achieve specific objectives. In this regard, failure paths, \ie ways abstracted from observable state–action sequences that culminate in unsuccessful task completion, provide signals that guide adversarial behavior without requiring direct access to the victim’s private rewards. Particularly, we introduce two common failure paths, which involve victim damage and task delay. The first measures the potential harm suffered by victim agents, while the second measures the potential obstacles encountered by the victim tasks, both accounting for the influence of adversarial agents at each step. Nevertheless, our reward design is extensible, allowing additional failure paths to be extracted from specific tasks and incorporated to further guide adversarial policy learning (see Section \ref{section_tech_overview}).

\noindent\ding{183}
After reward design, the computation of reward must be carefully specified to enable effective policy optimization by converting observed sequences of states and actions into quantitative step-wise signals that evaluate progress toward task objectives. Existing adversary reward computation methods\cite{guo2021adversarial} often assume access to the global state, that is, complete information about the environment and all agents, to ensure accurate and consistent reward estimation at each step. However, in multi-party open systems, such global state is typically unavailable due to factors such as perception distance limitation and region limitation~\cite{ma2024sub}, making precise step-wise reward computation difficult.

To address this challenge, we propose an estimation-based reward calculation model that leverages LSTMs to estimate the adversarial reward based on partial observations. Although adversarial agents do not have access to the complete global state, partial observations still provide task-relevant information, such as the adversarial agent’s own state and nearby environmental cues. By modeling sequences of these partial observations, LSTMs can capture temporal dependencies and accumulate information over time, effectively approximating the missing components of the global state. With sufficient training data linking partial observation sequences to overall task outcomes, the LSTM can learn an implicit mapping between partial observations and step-wise reward signals, providing informative feedback to guide adversarial policy optimization even under partial observability (see Section \ref{section_reward_shaping}).

\textbf{Evaluation.} 
Our method was evaluated on the Starcraft Multi-agent Challenge (SMAC) ~\cite{samvelyan2019starcraft} platform based on Starcraft \uppercase\expandafter{\romannumeral 2} ~\cite{SC2} and the autonomous driving simulation platform Highway-env ~\cite{highway-env}, which are both widely adopted for RL algorithm evaluations. We first evaluate the generalization effectiveness of our method across various task settings in multi-party open systems
. The experimental results demonstrate that our method is capable of launching generalizable adversarial attacks across these diverse task settings, resulting in respective reductions in winning rate of 96\%, 90\%, 87\%, 96\% on the corresponding Starcraft \uppercase\expandafter{\romannumeral 2} map "1m", "1c\_vs\_30zg", "8m", "MMM" and 80\%, 40\% in Highway-env scenario "highway", "intersection". 
Then, we compared the effectiveness and efficiency of our proposed estimation-based reward model with that of the traditional reward models. The experimental results show that our method outperforms the traditional model in terms of reducing the winning rate of the same well-trained victim agents on the same maps, achieving an average decline of 80\% vs. 20\%, 80\% vs. 15\%, 90\% vs. 90\% in map "8m", "MMM",  "6h\_vs\_8z", and 90\% vs. 25\%, 45\% vs. 30\% in scenario "highway" and "intersection". In terms of training efficiency, our reward model achieves significantly faster convergence, requiring only 2 million episodes compared to 18 million episodes for the traditional reward model on the "8m" map with two adversarial agents.
Next, we explored the effectiveness of deploying different numbers of adversarial agents across tasks of various difficulty levels. The results indicate that the more adversaries and the more difficult of victim tasks, the easier our attacks become effective. Finally, we also demonstrate that our attacks cannot be defended against by existing techniques through a few experiments. 

In summary, we make the following contributions:
\begin{itemize}[left=0pt]
\item  To the best of our knowledge, we are the first effort to attack DRL in multi-party open systems through adversarial policy learning.
\item  We propose a neutral agent-based adversarial policy learning approach to mislead well-trained victim agents without requiring direct interactions with them and full control over the environment.
\item  To implement our neutral agent-based approach, we redesign the reward functions by leveraging different failure paths. 
\item  We propose an estimation-based reward model to calculate reward without global states in each step. 
\item We evaluate the effectiveness and efficiency of our method on the SMAC platform and the autonomous driving simulation platform Highway-env.
\end{itemize}





\section{Problem Statement and Assumption}\label{section_problem}

\subsection{Problem statement}\label{subsec_problem}


\begin{figure*}[ht!]
\centering
\subfloat[Single-agent environment]{\includegraphics[width=0.51\linewidth]{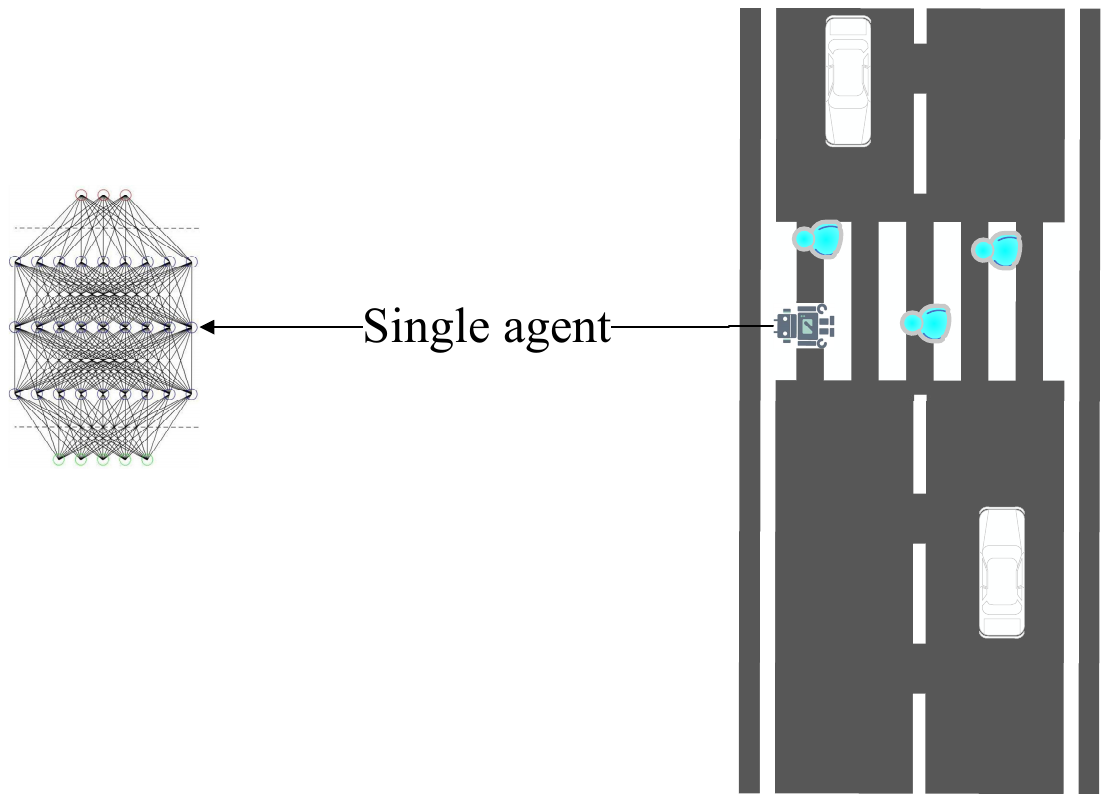}\label{fig1a}}
\subfloat[Two-agent competitive environment]{\includegraphics[width=0.51\linewidth]{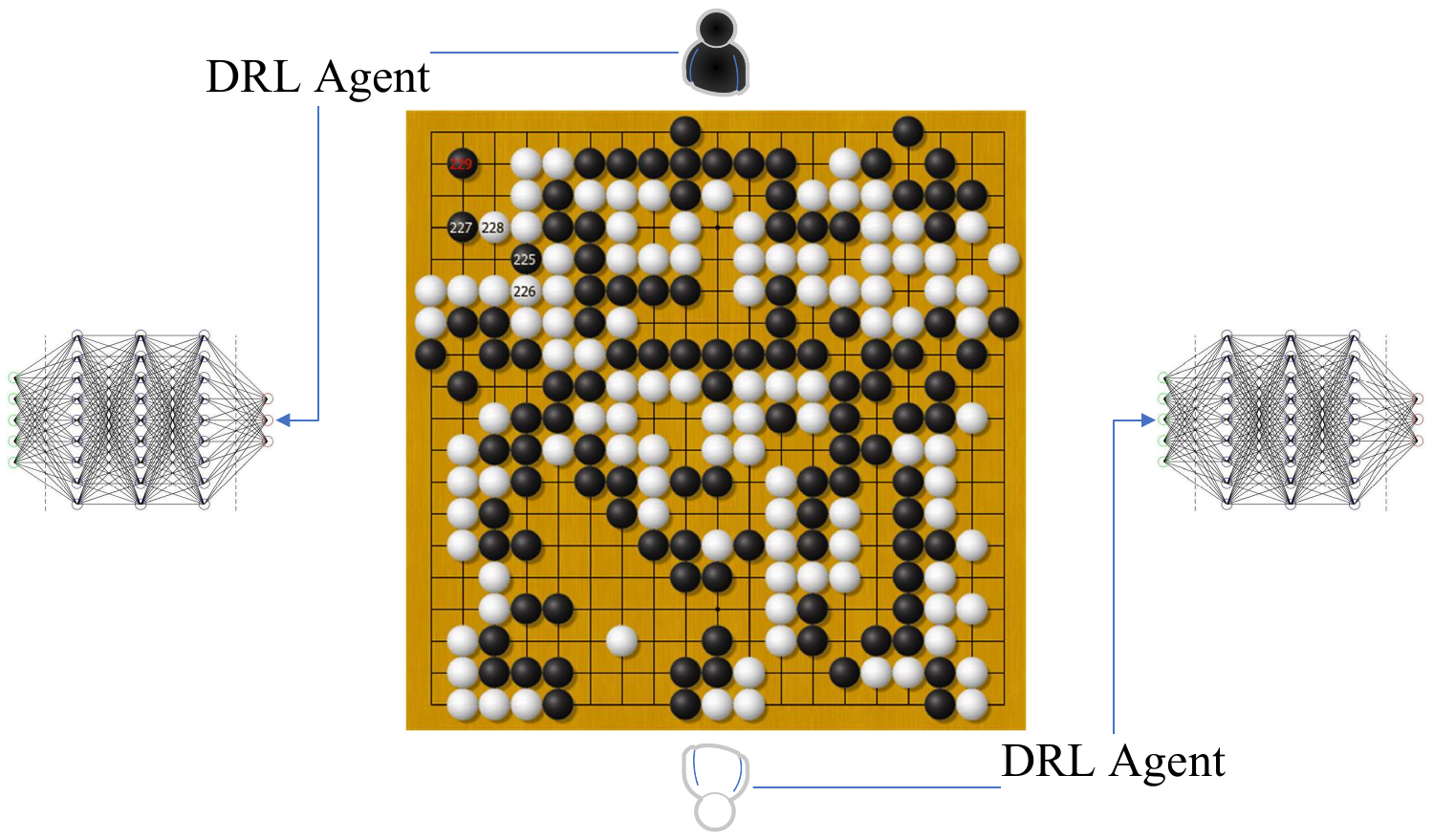}\label{fig1b}}
\vspace{-0.5cm}
\\
\subfloat[Multi-agent cooperative environment]{\includegraphics[width=0.51\linewidth]{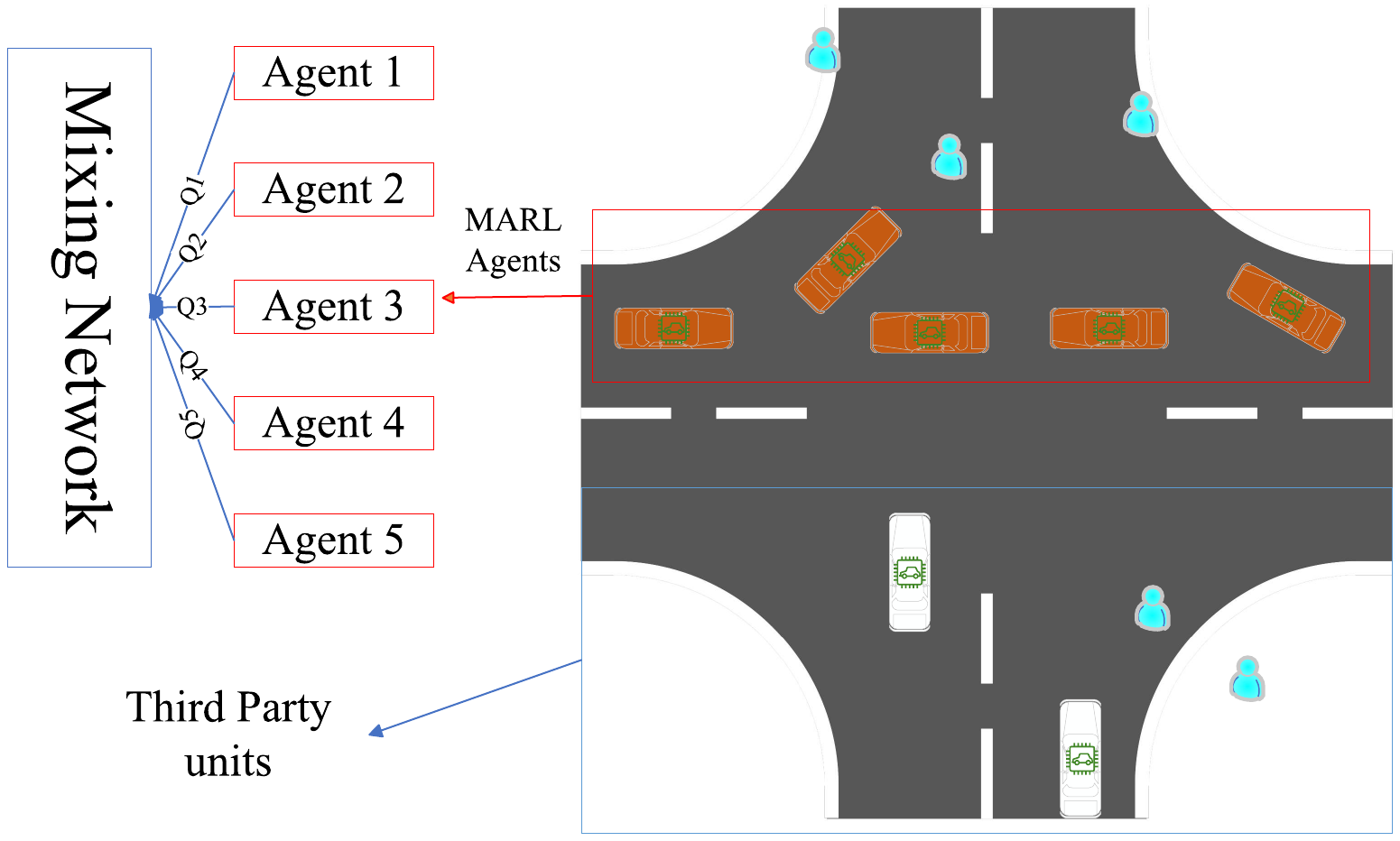}\label{fig1c}}
\subfloat[multi-party open system]{\includegraphics[width=0.51\linewidth]{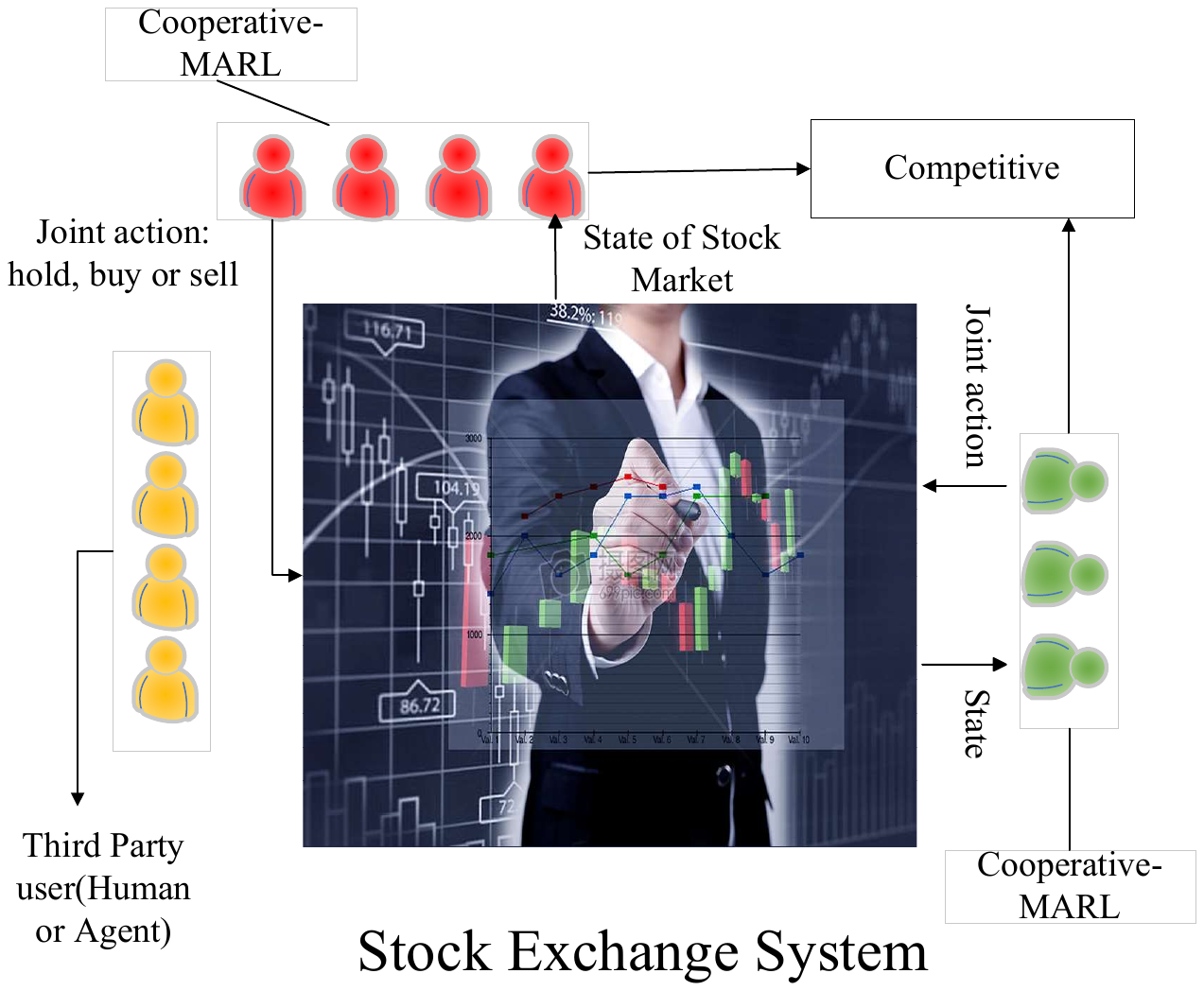}\label{fig1d}}
\caption{Categories of Reinforcement Learning Task Environments}
\label{fig:1}
\end{figure*}

Reinforcement learning refers to a type of algorithm capable of making optimal decisions in complex environmental changes to achieve target tasks. As shown in Figure \ref{fig:1}, in the reinforcement learning task environment, each RL agent can interact with the environment by observing its state and taking actions to update the state of the environment, and then receive a reward signal and the new observed state from the environment. 

In reinforcement learning, each agent ultimately aims to learn an optimal policy that guides all its actions to maximize the reward signals obtained from the environment, thereby enhancing task completion efficiency. In deep reinforcement learning, this optimal policy is typically derived from a well-designed deep neural network. The input of this neural network consists of the state information observed by the agent within the environment, while its output represents the probability distribution of actions the agent will take in that particular step.

However, the objectives for training optimal policies vary across different task environments. As illustrated in Figure \ref{fig:1}, reinforcement learning tasks can be categorized based on the number and relationships of agents. In single-agent tasks (shown in Figure \ref{fig1a}), only one agent interacts with the environment, aiming to accomplish a specific task, such as generating the thought process and response to a question. In two-agent competitive environments (shown in Figure \ref{fig1b}), two agents interact with the environment, with the goal of eliminating each other or contesting for victory. In multi-agent cooperative environments (shown in Figure \ref{fig1c}), multiple agents collaborate to interact with the environment, jointly fulfilling a target task through well-trained cooperation, such as in intelligent transportation systems. In multi-party open systems (shown in Figure \ref{fig1d}), multiple parties of agents, each consisting of at least one agent, interact with the environment. These groups can either collaborate, compete, or remain neutral.

In general, a two-agent competitive environment can be modeled as a single-agent environment ~\cite{gleave2019adversarial, wu2021adversarial, guo2021adversarial}, and both of them or a multi-agent cooperative environment can be considered as partial scenarios in multi-party open systems. That means if there is a way to effectively attack a designated party of agents to prevent them from achieving their goals by deploying a party of adversarial agents in the system with a neutral position from the attacker, there will be the same way in single-agent and two-agent competitive environments. In this research, therefore, we focus on a more general problem compared with previous research on adversarial policy, to develop a method that can attack any party of reinforcement learning agents in any circumstances in Figure 1. To be more specific, as shown in Figure \ref{fig:1}, in this work, we fix every agent in a multi-party open system except the adversarial party, and train an adversarial multi-agent party to collaboratively attack a designated party standing on the neutral position. 

\subsection{Assumption}\label{subsec_assumption}

Comparing with existing attacks on DRL, the changes in assumptions in our work are listed as follows: 

\begin{itemize}
    \item We assume only adversarial party agents adapt their policy in a multi-party open system immediately.
    \item This work does not assume that we can manipulate the environment or any agents of the victim party or those not belonging to the attacker.
    \item we assume that attacks occur only in open environments that allow deploying neutral agents at any time without directly participating in task of victim agents.
    \item We do not assume during every steps in an episode, adversarial agents and victim agents share global observations.
\end{itemize}

The further detailed discussion of assumptions can be seen in Appendix \ref{AP7}.

\subsection{Threat Model}

As discussed above, we describe the threat model from the perspectives of the envisioned attacker, threat surface, generality, and practicality. 

\textbf{Envisioned attacker}. We consider that the attacker aims at DRL applications such as autonomous driving, robots, and cyber security. Therefore, an attacker is supposed to be familiar with the DRL algorithm and the tasks listed above. The attacker may desire to fail the DRL tasks and benefit from it. For example, the attacker may desire to cause multiple autonomous driving crashes of a specific brand to achieve the purpose of commercial competition. 

\textbf{Threat surface}. Our attack can be readily deployed in any environment that is open or allows for the free deployment of neutral agents. Any task scenario within the environment can serve as an attack target, with no restrictions on the number or relationships of the target agents. 

\textbf{Generality}.Our proposed attack focuses on different types of DRL algorithms among single and multiple agents, value-based and policy optimization algorithms. Besides, as mentioned in Section \ref{subsec_problem}, the proposed method can attack different categories of DRL task environments. 

\textbf{Practicality}. As discussed in Section \ref{subsec_assumption}, to improve practical applicability, we have removed several assumptions, including environmental manipulation, global state observability, and direct participation in the victim's task execution. 

\section{Background}\label{section_background}

Recent proposed DRL methods can be categorized into Q-learning based algorithms (e.g. ~\cite{van2016deep, mnih2015human, rashid2020monotonic}) and policy optimization algorithms (e.g. ~\cite{schulman2017proximal, yu2022surprising}). Given that our proposed attack method does not restrict the quantity of adversarial agents, we model our problem as a Multi-Agent Reinforcement Learning(MARL) task. Among these recent RL methods above, QMIX is one of the best-performance and widely used algorithms in solving MARL tasks. Therefore, in this work, we take QMIX as an example to show our proposed method in Section \ref{section_tech_overview}, while the method is also available for other algorithm frameworks such as VDN and MAPPO. In this section, we first summarize the main algorithms of DRL and MARL, then briefly show how to model a MARL problem formally, and finally introduce the QMIX structure. 

\subsection{Deep Reinforcement Learning}

Deep Reinforcement Learning (DRL) represents a significant paradigm shift within the broader field of Reinforcement Learning, fundamentally distinguished by its integration of deep neural networks as powerful, high-capacity function approximators. Whereas traditional Reinforcement Learning approaches typically rely on explicitly designed, often linear or tabular, methods (such as state aggregation, tile coding, or linear value function approximation) to handle the value function or policy representation, DRL leverages the representational power of deep learning to automatically discover intricate hierarchical features directly from high-dimensional, raw sensory inputs, such as pixels in images or complex sensor streams. DRL architectures employ deep neural networks as universal nonlinear function approximators, which enables agents to learn abstract representations end-to-end, effectively scaling RL to previously intractable domains with high-dimensional perceptual inputs (e.g., playing Atari games from pixels ~\cite{mnih2015human, mnih2013playing}, robotic control from vision ~\cite{andrychowicz2020learning, masmitja2023dynamic, radosavovic2024real}, complex strategy games like Go ~\cite{silver2016mastering}). Recall that we will take QMIX as the example in Section \ref{section_tech_overview}, following we thus focus on the introduction of Q-learning based algorithms. 

\textbf{Q-learning based algorithms}. Q-learning-based algorithms represent a prominent approach in deep reinforcement learning (DRL), addressing sequential decision-making problems through value function approximation. Rooted in the principles of temporal difference (TD) learning, these methods estimate action-value functions \(Q(s,a)\) to determine optimal policies by maximizing expected cumulative rewards. Classical Q-learning employs a tabular representation to iteratively update Q-values through the Bellman equation:
\begin{equation}
  Q(s,a)=R(s,a)+\gamma\sum_{s'}P(s'|s,a)V_\pi(s'), \label{s3_1}
\end{equation}
where \(s,a,s'\) represent the current state, the action taken by the policy, and the state of the next step, respectively. \(V_\pi(s')\) is the bellman equation of state \(s'\), denoted as 

\begin{equation}
  V_\pi(s') = \sum_a \pi(a|s')\sum_{s''}p(s''|s',a)[R(s', a)+\gamma V_\pi(s'')].
  \label{s3_1_1}
\end{equation}

However, the advent of deep neural networks has led to pivotal advancements through Deep Q-Networks (DQN), which parameterize Q-functions via deep learning architectures to handle high-dimensional state spaces. 


\subsection{Multi-Agent Reinforcement Learning}

Multi-agent reinforcement learning (MARL) extends traditional reinforcement learning paradigms to settings where multiple autonomous agents interact within a shared environment, necessitating coordination, competition, or hybrid objectives. Unlike single-agent systems, MARL addresses unique challenges arising from nonstationarity - where an agent's optimal policy depends on the evolving behaviors of other agents - and partial observability, often requiring decentralized decision-making under imperfect information. 

\subsubsection{Modeling a MARL problem}

Given a MARL task with continuous action space, it is usual to model the task as a Decentralized Partially Observable Markov Decision Process (Dec-POMDP), which contains the following components:


\begin{itemize}[left=0pt]
\item  a finite set of agents 
\begin{math}
    N = \{1,\dots,n\}
\end{math} , each of agent follows an independent policy. 
\item a finite set of individual state \(S_i\) for each agent 
\begin{math}
    i \in N
\end{math}. In each \(S_i\) includes an state \begin{math}s(i, t)\end{math}, where each \begin{math}s(i, t)\end{math} represents the state of agent \begin{math}i\end{math} in time \begin{math}t\end{math}. Global states can be described with all of individual state \begin{math}S_i\end{math}
\item a finite joint action set \begin{math}A\end{math}, where each joint action \begin{math}A_t\end{math} refers to the joint action in time \begin{math}t\end{math}. Each joint action is composed of \begin{math}a(i, t)\end{math} for each agent \begin{math}i \in N\end{math}.
\item a global state transition function \begin{math}P : S \times A\rightarrow S\end{math}, where \begin{math}P(s’| s, a)\end{math} denotes the probability that the global state s transits to \begin{math} s '\end{math} by taking joint action \begin{math}a\end{math}. 
\item a reward function \begin{math}R_i : S \times A_i \rightarrow R\end{math}, where \begin{math}r(i, s, a)\end{math} indicates the expected reward that agent i will receive after taking action a at state s. 
\item a discounted rate \begin{math}\gamma\in[0,1]\end{math}, which is usually multiplied by future reward.
\item a finite set of policies \begin{math}\pi_i : S_i \rightarrow A_i\end{math} for each agent \begin{math}i \in N\end{math}, where \begin{math}\pi_i(a_i|s_i)\end{math} refers to the probability distribution of action taken by agent \begin{math}i\end{math} at state \begin{math}s_i\end{math}.
\end{itemize}
The final target of MARL is to learn an optimal set of policies \begin{math}\pi_i(a_i|s_i)\end{math} for each agent \begin{math}i \in N\end{math} that could maximize the expectation of the state value function \begin{math}V_\pi^{tot}(s)\end{math} or state-action value function \begin{math}Q_\pi^{tot}(s,a)\end{math} over a sequence of actions generated through the policy.

\subsubsection{QMIX}

QMIX is a commonly used algorithm in current MARL tasks for addressing multi-agent reward allocation problems. It proposes a decentralized greedy strategy to ensure that globally optimal joint action A is equivalent to the combination of each optimal action \begin{math}a_i\end{math} of agent \begin{math}i \in N\end{math} individually:
\begin{equation}
  argmax_AQ^{tot}(s,A)=\left(
  \begin{aligned}
      &argmax_{a_1}Q_1(s_1,a_1)\\
      &argmax_{a_2}Q_2(s_2,a_2)\\
      &...\\
      &argmax_{a_n}Q_n(s_n,a_n)
  \end{aligned}
  \right).\label{s3_2}
\end{equation}

QMIX transforms Equation (\ref{s3_2}) into the monotone constraint for each \begin{math}Q_i\end{math} using a mixing network. Equation (\ref{s3_2}) holds only if the following monotonicity is satisfied: 

\begin{equation}
  \frac{\partial Q^{tot}}{\partial Q_i}\geq 0\label{s3_3}
\end{equation}

In the structure of QMIX, each agent holds an independent deep Q network to calculate the independent Q value using the mixing network \begin{math} F\end{math} with Equation (\ref{s3_1}):

\begin{equation}
  Q^{tot}=F(Q_1,Q_2,...,Q_n)\label{s3_4}
\end{equation}

Similarly to DQN, QMIX trains an end-to-end model with the following loss function under batch size b:

\begin{equation}
  \begin{aligned}
      &L(\theta) = \sum_{i=1}^{b}[(y_i^{tot}-Q^{tot}(s,a;\theta))^2]\\
      &y^{tot}=R^{tot}+\gamma _{max_{a'}}Q^{tot}(s',a';\theta^-))
  \end{aligned}\label{s3_5}
\end{equation}

\section{Technique Overview}\label{section_tech_overview}

Recall that we attack a set of well-trained victim agents by training a set of neutral agents. To achieve this, as discussed in Section \ref{section_problem}, we do not assume the attacker has access to the models of victim agents (including observation, action, reward function and other module of victim agents) nor the global state at each step in the training phase. Rather, we assume that the results and status of all victim agents are available at the end of one episode. In this section, we first display the reward function design of our method. Then, we briefly specify how to build the objective function with new designed reward function to extend a MARL algorithm and thus implement our attack method at a high level. 

\begin{figure*}
\centering
\includegraphics[width=1\linewidth]{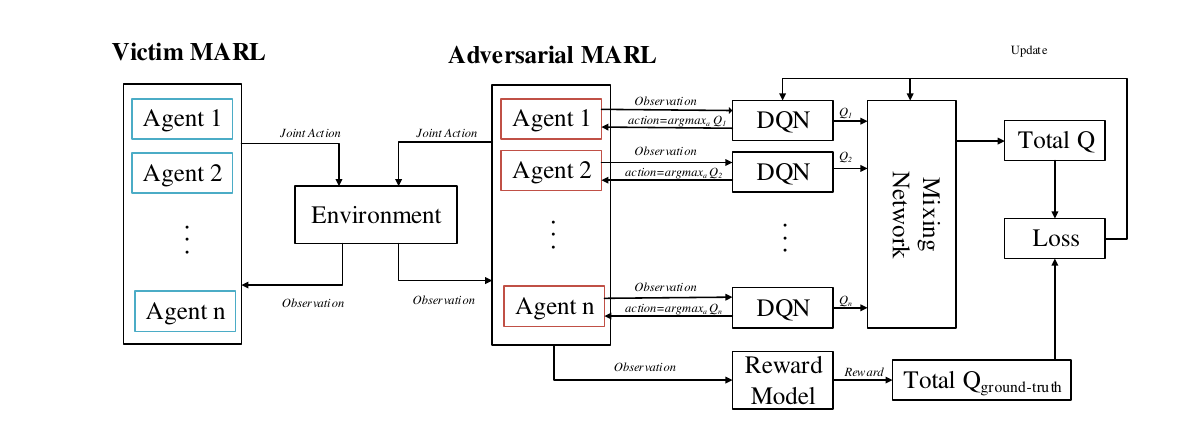}
\caption{The algorithm framework of our proposed method}
\label{fig:2}
\end{figure*}

\subsection{Problem Definition}

Following early research on MARL ~\cite{rashid2020monotonic} mentioned in \ref{section_background}, we also formulate a multi-party open system as a Dec-POMDP, represented by \(M=<(\mathcal{N_\alpha, N_v}), \mathcal{S}, (\mathcal{A_\alpha, A_v}), \mathcal{P}, (\mathcal{R_\alpha, R_v}), \mathcal{\gamma}>\). Here, \(\mathcal{N_\alpha}\) and \(\mathcal{N_v}\) refer to the agent set of adversaries and victims separately. \(\mathcal{S}\) denotes the global state set. \(\mathcal{A_\alpha}\) and \(\mathcal{A_v}\) are the joint action sets for adversarial agents and victim agents, respectively. \(\mathcal{P}\) represents a joint state transition function \(\mathcal{P}: \mathcal{S} \times \mathcal{A_\alpha} \times \mathcal{A_v} \xrightarrow{} \Delta (\mathcal{S})\). As mentioned in Section \ref{section_background}, the state transition is a stochastic process, thus we use \(\Delta (\mathcal{S})\) to denote a probability distribution on \(\mathcal{S}\). The reward function can be represented as \(\mathcal{R_i}: \mathcal{S} \times \mathcal{A_\alpha} \times \mathcal{A_v} \xrightarrow{} \mathbb{R}; i \in \{\alpha, v\}\). 

In this paper, following previous research on adversarial policy learning ~\cite{guo2021adversarial}, we assume that agents except adversaries follow fixed policies. Holding this assumption, our problem can be viewed as a single-party Dec-POMDP for adversarial agents, denoted by \(M_\alpha{}=<\mathcal{N_\alpha}, \mathcal{S}, \mathcal{A_\alpha}, \mathcal{P_\alpha}, \mathcal{R_\alpha}, \mathcal{\gamma}>\). Note that the state transition function \(\mathcal{P_\alpha}\) and global state \(\mathcal{S}\) here are not available in explicit form. Instead, each of the agents can get only its own observation \(s_i; i \in \{1,2,\dots , n\}\) where \(n\) represents the total number of all agents in the environment. 

\subsection{Reward function design}

Defining and calculating the reward for adversarial agents presents a significant challenge in our work. Previous studies on adversarial policy training define the reward for adversarial agents as the gain they achieve in the task environment ~\cite{gleave2019adversarial} (in a zero-sum competition) or as their own gain minus the reward of victims ~\cite{guo2021adversarial} (in non zero-sum competition). Under zero-sum conditions, the former and latter definitions are equivalent.

Admittedly, it is simple to use a direct reward function such as the negative of the victim agents' reward. However, as mentioned in Section \ref{section_problem}, we do not assume the victim agents' reward design is available for the attacker. As such, we design a different and universal reward model to fulfill our objective as follows. Recall that no matter attacking single agent RL tasks or multi-agent RL tasks, the final target of adversarial agents is to mislead victims to a failure ending, which can be caused by different ways in different tasks. Therefore, as a specific example displayed in Appendix \ref{AP5}, to fail a set of well-trained victim agents, we redesign the adversarial reward function by further exploring as many paths that lead victims to failure as we can. Mathematically, the newly designed reward model can be written as:

\begin{equation}
  R_{j}^{i} = \{R_1^i, R_2^i, ..., R_n^i\}, \label{s4_1}
\end{equation}

\noindent where \(i, j\) refer separately to the adversarial agent \(i\) and the failure path \(j\). With this practice, from the victim agents' viewpoints, they will be misguided to suboptimal decisions and thus reduce the task success rate. 

Considering the distinction of different failure paths, we further explore how much impact can be caused by each failure path. With different impacts on victim tasks, we manually define a weight vector, measuring the importance and effectiveness of each corresponding failure path: 

\begin{equation}
  W_{j} = \{W_1, W_2, ..., W_n\}, \label{s4_2}
\end{equation}

\noindent where \(j\) refers to the failure path \(j\). With this weight vector, the reward function of the adversarial agent \(i\) can be formally rewritten as:

\begin{equation}
  R_{\alpha}^{i} = W \times (R^i)^\top, \label{s4_3}
\end{equation}

\noindent where \(R_{j}^{i}\) and \(W_{j}^{i}\) are defined by Equations (\ref{s4_1}) and (\ref{s4_2}). Note that the weight vector varies between different tasks, while it is simple to make a quick configuration before facing a specific task. Similarly, the total reward of adversaries can be defined as follows: 

\begin{equation}
  R_{w}^{tot} = W \times (R^{tot})^\top, \label{s4_3_1}
\end{equation}

\noindent where \(R^{tot}\) is a vector defined as \(R^{tot}_{j} = \{R^{tot}_{1}, R^{tot}_{2}, \dots, R^{tot}_{n}\}\).

\subsection{Objective function Building}

As introduced above, we design a new reward function to measure the effectiveness of adversaries in each step. However, a short-term feedback signal is not enough for an agent, thus we further extend the existing RL algorithm by building an objective function to provide long-horizon measurement for adversaries with a newly designed reward. In previous RL research, it is common to build an objective function with the Value function ~\cite{sunehag2017value} or Q-value function ~\cite{rashid2020monotonic}. Considering the popularity and effectiveness of the QMIX algorithm in MARL ~\cite{samvelyan2019starcraft}, we take the QMIX algorithm as an example to show how to build an objective function based on the Q-value function and train adversarial agents with the newly designed reward function as follows. 

As introduced in Section \ref{section_background}, the objective of QMIX algorithms is to maximize the expected return calculated by the total Q-value function consisting of the individual Q-value of each agent. Therefore, the objective function of our proposed algorithm can be represented as:

\begin{equation}
  J(\theta)=maximum_{A{^{adv}}} Q_{adv}^{tot}(S,A). \label{s4_4}
\end{equation}

To maximize the total Q-value function, the independent Q-value should be defined first following Equation (\ref{s3_1}) in DQN reformulated by our newly designed reward function:

\begin{equation}
  Q^{\pi}_i(S,A_i)=R_i^\alpha(S,A_i)+\gamma \sum _{S'}P(S'|S,A_i)V^{\pi}_i(S').
  \label{s4_5}
\end{equation}

Here, \(\pi\) refers to the joint policy of adversaries, victims, and other potential agents deployed in the environment. \(R_i^\alpha(S,A_i)\) follows Equation (\ref{s4_3}) where \(i\) indicates the index of each adversarial agent. As mentioned in Section \ref{section_problem}, we assume only adversarial agents adapt their policies in our proposed attack. Under this setup, we have the following proposition (see proof in Appendix \ref{AP1}). 

\begin{proposition}
    \label{prop_1}
    In a multi-party open system, if all agents follow fixed policies except agents of one specific party, the state transition of the environment system will depend only upon the joint policy of agents belonged to this specific party rather than the joint policy of all agents in the system.
\end{proposition}

With the proposition above, we can redefine the independent Q-value of adversarial agents below. 

\begin{equation}
  Q^{\pi^\alpha}_i(S,A_i^\alpha)=R_i^\alpha(S,A_i)+\gamma \sum _{S'}P(S'|S,A_i^\alpha)V^{\pi^\alpha}_i(S').
  \label{s4_6}
\end{equation}

Here, \(V^{\pi^\alpha}_i(S')\) is calculated following Equation (\ref{s3_1_1}) by replacing the reward function into \(R_i^\alpha(S,A_i)\) defined in Equation (\ref{s3_4}). As shown above, the new independent Q-value function no longer encloses the policies, observations, or actions of victims nor other agents not belonging to adversaries.  It perfectly addresses the concern about the necessity of victim agents. 

Recall that our reward function is a weighted sum of each victim failure path. The relationship between the weight vector and the independent Q-value function is described in the following proposition (see the proof in Appendix \ref{AP2}). 

\begin{proposition}
    \label{prop_2}
    The long-horizon expected return shares the same weighted changes with the short-term reward function: 
\end{proposition}

\begin{equation}
  Q^{\pi^\alpha}_i(S,A_i^\alpha)=W \times (R_i(S,A_i)+\gamma \sum _{S'}P(S'|S,A_i^\alpha)V_i(S'))^\top . 
  \label{s4_7}
\end{equation}

Here, vectors \(R_i\) and \(W\) are defined following Equations (\ref{s4_1}) and (\ref{s4_2}). \(V_i(S')\) is a value vector calculated with \(R_i\) following the Equation (\ref{s3_1_1}). From Proposition \ref{prop_2}, it was observed that the weight vector remained unchanged during the computation of long-term returns, thereby eliminating concerns about the weight vector's potential misalignment across long-term returns and immediate rewards, and the necessity of choosing between these objectives. 

As introduced in Section \ref{section_background}, the total Q-value is computed using an option consisting of each independent Q-value. With respect to Equations (\ref{s3_5}) and (\ref{s4_3_1}), the total Q-value can be approximated by: 

\begin{equation}
        Q^{tot}\approxeq R^{tot}_w+\gamma\sum_{S'}P(S'|S,A)V_{\pi}(S').
   \label{s4_8}
\end{equation}

Similarly, following Proposition \ref{prop_1} and \ref{prop_2}, Approximation (\ref{s4_8}) can be rewritten as: 

\begin{equation}
        Q^{tot}\approxeq W \times (R^{tot}+\gamma\sum_{S'}P(S'|S,A^\alpha)V_{\pi^\alpha}(S')).
   \label{s4_9}
\end{equation}

With this total Q-value and independent Q-value, we can train the adversaries by Equations (\ref{s3_4}) and (\ref{s3_5}). Finally, the objective function can be denoted as follows.

\begin{equation}
  J(\theta)=maximum_{A{^{\alpha}}} Q_{\alpha}^{tot}(S,A).
  \label{s4_10}
\end{equation}

\section{Reward Calculation}\label{section_reward_shaping}

In Section \ref{section_tech_overview}, we discuss the entire structure of the technique (as shown in Figure \ref{fig:2}) and redefine the adversarial reward function. In this section, we will provide further details on how to calculate the adversarial reward. More specifically, we will introduce an estimation-based reward model updated by a rule-based method for training adversarial party agents.

\begin{figure*}
\centering
\includegraphics[width=1\linewidth]{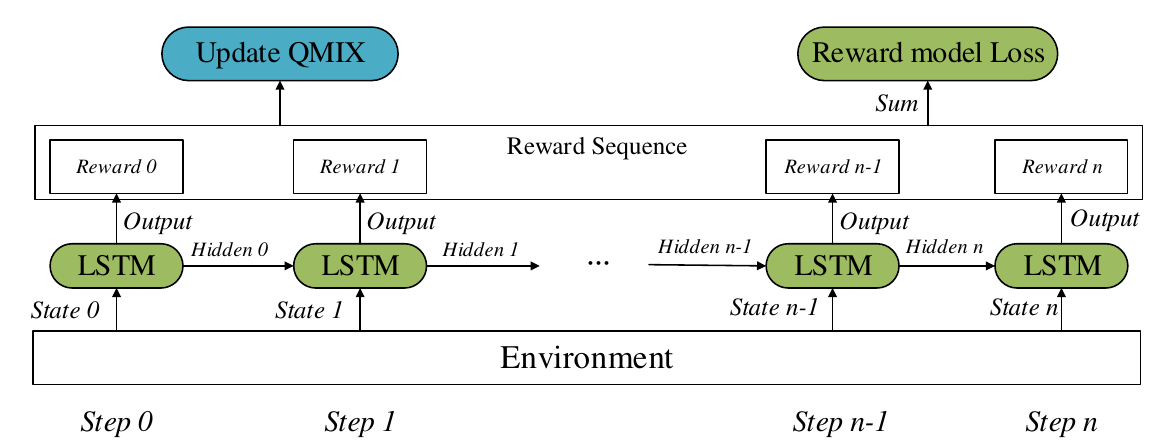}
\caption{Estimation-based reward model framework}
\label{fig:3}
\end{figure*}

\subsection{Rule-based reward calculation}

Rule-based reward calculation is a kind of usual and common method to calculate reward for most reinforcement learning. Recall that adversarial reward is defined by the victim's failure paths. Under this setup, we can now calculate adversarial reward with the following rules:

\noindent\textbf{Rule 1}: Final reward signal only occurs at the last step and only depends on whether the target mission is completed. 

\noindent\textbf{Rule 2}: If target mission is completed by victim agents, the reward at the last step will be 0.

\noindent\textbf{Rule 3}: If target mission is not completed by victim agents, a great reward signal will occur at the last step.

\noindent\textbf{Rule 4}: If the rule-based method is the baseline in evaluation, the
immediate reward signal is depended on the process of each corresponding failure path computed by Equition (\ref{s4_3_1}).

\noindent\textbf{Rule 5}: If the rule-based method is the ground truth of the estimation-based reward model, there will be no any immediate reward. 

With Rules above, adversarial agents are now able to reach positive reward easily and do not need knowing reward function of victim agents in advance.

\subsection{Estimation-based reward model}

Although rule-based methods have already solved most problems in reward shaping, yet still one problem remains. If using a rule-based method as our reward model to calculate immediate rewards in each step, we still have to follow the assumption that the adversarial party agents shall obtain global states in both the training and evaluation phases, which is not practical in a cooperative environment as mentioned in section \ref{section_intro} and section \ref{section_problem}. To remove this unpractical assumption, we provide another solution of reward shaping. 

Without global states, adversarial reward cannot be directly calculated. In this case, we design a neural network to estimate adversarial reward in each step, which takes the partially observed state of the adversarial party as input and gives an estimated reward of the total reward of the adversarial party in each step. While global state is not available at each step, adversarial party can only approach the result of the victim mission after the last step. Therefore, the ground truth cannot be designed for each output of the network in each step. Instead, we calculate ground truth by using the rule-based method with the result of the victim mission after the last step. Considering the state at each step is time sequence data, we design this network based on Long Short-Term Memory (LSTM) algorithm and store each output until the last step as shown in Figure \ref{fig:3}. Under these information, adversarial reward of each step can be estimated by calculating loss with distance between the sum of every reward of each step and ground truth:
\begin{equation}
  L_{R-model}(\theta)=(R^{tot}_{GT}-\sum _{i=1}^{l}M_i(obs))^2,\label{16}
\end{equation}
where \(M\) represents LSTM network, \(l\) is the length of this episode and \(obs\) refers to the partial observation of adversarial-party agents. With estimation-based reward model, we can finally display the algorithm of our proposed attack method as shown in Appendix \ref{AP3}.

\section{Evaluation}\label{section_evaluation}

In the evaluation, we aim to answer the following research questions:
\begin{itemize}[left=0pt]
    \item \textbf{RQ1:} What is the generalization effectiveness of our neutral agent-based method across different scenarios in multi-party open systems? 
    \item \textbf{RQ2:} What is the performance of our estimation-based reward model, compared with  that of the traditional reward model and the direct use of the rule-based calculation method as the reward model?
    \item \textbf{RQ3:} What is the influence of varying numbers of adversarial agents on victim agents?
    \item \textbf{RQ4:} How effective is our neutral agent-based method in various difficulty-level tasks?
    \item \textbf{RQ5:} How effective is our neutral agent-based method against simple defenses by a single round of adversarial retraining and other existing countermeasures?
\end{itemize}

\subsection{Experiment setup}

Our experiments are deployed on the Starcraft \uppercase\expandafter{\romannumeral 2}-based intelligent agent testing platform - the Starcraft Multi-Agent Challenge (SMAC), and the automonous driving simulation Highway-env, which are both widely adopted platforms for evaluating reinforcement learning algorithms. 
There are three primary reasons for selecting Starcraft \uppercase\expandafter{\romannumeral 2} and SMAC as our experimental platforms. Firstly, as discussed in Section \ref{section_problem}, our experiments should be conducted in a more general setting to ensure the applicability of our attack method across various scenarios. Starcraft \uppercase\expandafter{\romannumeral 2} is built on a multi-party open system, aligning with our assumptions. Secondly, SMAC provides an open-source and convenient interface for Starcraft \uppercase\expandafter{\romannumeral 2}, along with flexible map designs and unit attribute configurations, enabling us to design diverse scenarios for testing and comparing our attack methods. Thirdly, as a commonly used Multi-Agent Reinforcement Learning (MARL) testing platform, SMAC has hosted numerous experiments involving various RL methods and attacks on reinforcement learning, allowing us to conveniently select baseline methods for comparison. Below, we will briefly introduce the Starcraft \uppercase\expandafter{\romannumeral 2} game environment, the agent configuration on the SMAC platform, and the evaluation metrics.

Similarly, the reasons for selecting Highway-env as our evaluation environments are as follows. First, we seek to explore the performance of our adversarial deployment in semi-realistic task scenarios rather than limiting it exclusively to gaming environments. Second, the Highway-env framework modularizes decision-making tasks in autonomous driving contexts, disentangling them from perception layers and other components irrelevant to DRL decision processes. This isolation enables a granular examination of how our proposed adversarial attacks mislead the decision-making mechanisms of autonomous driving agents. Finally, Highway-env provides heterogeneous task scenarios such as Highway, Intersection , and Merging, which structurally align with the experimental requirements articulated in our research questions. 

The further detailed introductions of Starcraft \uppercase\expandafter{\romannumeral 2}, SMAC and Highway-env can be seen in Appendix \ref{AP4}.

\textbf{StarCraft \uppercase\expandafter{\romannumeral 2} Maps}. To better validate 5 research questions above, we deploy different attacks through multiple maps. SMAC platform includes some of them, the rest are designed by us using StarCraft \uppercase\expandafter{\romannumeral 2} Map Editor. Following, we briefly introduce the maps in our experiments. 

\begin{itemize}[left=0pt]
    \item \textbf{Map A: "1m"}. This map is designed by us with 1 Marine for each party, with two competitive parties separately controlled by the victim agent and the PC agent. The task for victim agent is to kill unit Marine of its competitive party. We deploy neutral Marines as our adversarial agents in this map. 
    \item \textbf{Map B: "1c\_vs\_30zg"}. This map is designed by us with 1 Colossi for victim and 30 Zerglings for victim's task. The task object is eliminating all Zerglings in limited steps. We deploy neutral Colossi as our adversarial agents in this map. 
    \item \textbf{Map C: "8m"}. This map is contained in SMAC map list with 8 Marines for each party with two competitive parties separately controlled by victim agents and PC agents. The task for victim agents is to kill all Marines of its competitive party. We deploy neutral Marines as our adversarial agents in this map. 
    \item \textbf{Map D: "MMM"}. This map is contained in the SMAC map list with 1 Medivac, 2 Marauders, and 7 Marines for each party, with two competitive parties separately controlled by the victim agent and the PC agent. The task for victim agents is to kill all units of its competitive party. We deploy neutral Marines as our adversarial agents in this map. 
    \item \textbf{Map E: "6h\_vs\_8z"}. This map is contained in SMAC map list with 6 Hydralisks for victim and 8 Zealots as victim's task.The task object is eliminating all Zerglings in limited steps. We deploy neutral Hydralisks as our adversarial agents in this map. 
\end{itemize}

\textbf{Highway-env scenarios}. As an integral component of our experimental framework, we deploy the attack across diverse autonomous driving simulation scenarios on the Highway-env platform. Highway and intersection are autonomous driving simulation environments within the Highway-env platform, composed of a straight-line highway or intersection, controllable vehicles, and other vehicles. In these environments, we can construct specific scenarios by adjusting the number of controllable vehicles and other vehicles. Below, we present a concise overview of the scenarios we used in experiments.

\begin{itemize}[left=0pt]
    \item \textbf{highway\_M}.  This scenario consist of three victim agents, three adversarial agents, and two other vehicles in highway.
    \item \textbf{intersection\_S}. This scenario consist of one victim agent, three adversarial agents, and two other vehicles in intersection.
\end{itemize}

\subsection{RQ1: Generalization effectiveness of our neutral agent-based method}

In this section, we will prove the effectiveness of our method across various environments designed to describe each circumstance in section \ref{section_problem}. Specifically, in this experiment design, we separately train adversarial agents in maps A, B, C, D in StarCraft \uppercase\expandafter{\romannumeral 2} and scenarios highway\_M, intersection\_S and observe the convergence status. Each map corresponds to different circumstances of the victim described in section \ref{section_problem}. Map B is a single-agent task, A and intersection\_M  are two-agent competitive tasks, and C, D, highway\_M, are cooperative MARL tasks. All these maps can occur in a multi-party open system. The deployed attack results are displayed in Table \ref{tab:0}.



As shown in Table \ref{tab:0}, our proposed method can decrease the win rates from at least 95\% to at most 10\% across every map and scenario, which proves the generalization of our proposed attack method that can deploy an effective attack under each scenario in multi-party open systems we discussed in section \ref{section_problem}. 

\begin{table}
    \centering
    \begin{tabular}{c|c|c} \hline
        \textbf{Winning rate}& \textbf{Under attack}  & \textbf{No attack} \\\hline
        \textbf{Map "1m"}& 0.04 & 1.0 \\\hline
        \textbf{Map "1c\_vs\_30zg"} &0.1  &1.0 \\\hline
        \textbf{Map "8m"} & 0.08 & 0.95 \\\hline
        \textbf{Map "MMM"} & 0.0 & 0.96  \\\hline
        \textbf{"highway\_M"} & 0.12 & 1.0  \\\hline
        \textbf{"intersection\_S"} & 0.28 & 0.72  \\\hline
    \end{tabular}
    \caption{Wining rate of victim agents with and without attack by our proposed method}
    \label{tab:0}
\end{table}

\subsection{RQ2: Performance of our estimation-based reward model}

Recall that we have proposed two reward shaping methods above, where the estimation-based model is updated by the rule-based method. Reward shaping is an important component of our proposed method, thus, in this experiment, we aim to validate the performance of the estimation-based reward model with baselines. Noted that even calculating immediate reward by the rule-based method is unpractical, as discussed in section \ref{section_reward_shaping}, we still would like to treat it as an alternative baseline model compared to the estimation-based model. Besides, we also introduce a traditional reward model designed by ~\cite{gleave2019adversarial, wu2021adversarial} as another baseline. Within the proposed framework, we separately train the multi-adversarial policies with the estimation-based reward model, rule-based model, and the traditional model while keeping other components fixed. We compare the convergence during training and the attack effectiveness of the trained adversarial agents to evaluate the impact of different reward shaping designs.

In this experiment, we choose the map C, D, E and scenario highway\_M, intersection\_S to mainly validate how our proposed reward models perform in complex multi-party open systems, especially the estimation-based reward model, which is specially designed for these complex environments. 

As displayed in Figure \ref{fig:6}, while attacking victim agents under a relatively easy task (map C and highway\_M), the estimation-based reward model shows faster convergence speed and greater attack effectiveness compared to both the rule-based reward model and baseline model. During the attack on victim agents under a medium difficulty task (map D), estimation-based reward model performs similarly to the rule-based model and better than the baseline. And all reward models share similar effectiveness attacking the victim agents under difficult tasks. With the observation above, we carefully conclude that the estimation-based reward model is proved to be more general and effective.

\subsection{RQ3: Influence of varying numbers of adversarial agents}

In our proposed method, we have introduced multi-agent reinforcement learning into the adversarial policy training framework. Therefore, this experiment aims to observe the impact of changing adversarial agent quantity from single to multiple on both training and attack deploying stages. 

In this experiment, we change the quantity of adversarial agents while training on map C, D, E and keeping other components fixed. In map C, we mainly validate the collaboration capability of adversarial agents in attacking victims under a well-matched competitive environment. In map D, we intend to verify the effectiveness of the multi-agent adversarial policy against victim agents that cooperatively control multiple kinds of units to achieve their task. In map E, we try to demonstrate that the multi-agent method performs better as well when attacking victim agents under very difficult tasks. 

\begin{figure*}
    \centering
    \includegraphics[width=1\linewidth]{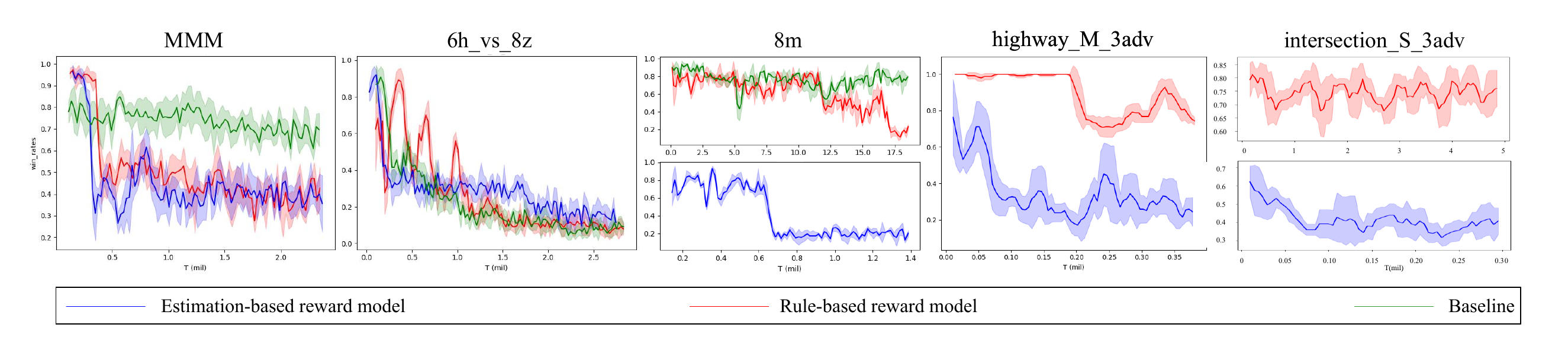}
    \caption{Comparison of wining rates trend during training adversarial agents across different reward model in Starcraft \uppercase\expandafter{\romannumeral 2} maps}
    \label{fig:6}
\end{figure*}



As shown in Appendix 
, multi-agent collaboratively attacking displays much more effectiveness than the single agent method in each map. To some extent, the more adversarial agents attacker deploys, the more effective the attack is. However, deploying too many adversarial agents will potentially increase the risk of being detected. Therefore, attacker should consider a specific number of adversarial agents in the light of specific conditions. 

\subsection{RQ4: Effectiveness in various difficulty level tasks}


Considering that victim agents may have varying sensitivities to adversarial attacks in tasks with different levels of complexity, we designed this experiment to investigate the impact of the complexity of the victim task on the attack effectiveness. Using the same parameters and algorithm, we implement attacks on well-trained victim agents in multiple SMAC maps with varying difficulty levels. We compare and observe the convergence during adversarial agent training and the resulting attack effectiveness to assess how task complexity influences the attack outcomes.

As in 
Appendix \ref{AP6} and 
Figure \ref{fig:6}, adversarial agents take very long period episodes of training to map C (the easy victim task) and train more efficiently in map D (the medium difficult victim task) and E (the difficult victim task). Besides, adversarial agent training shows more stable in map E compared with map D. Under these discoveries, we conclude that victim agents under more difficult tasks show more vulnerability to our proposed attack method.

\subsection{RQ5: Effectiveness against countermeasures}

Our proposed attack method is very difficult to defend due to its features such as unpredictable, hard to detect, and varying number of adversarial agents. However, we still intend to know whether our attack can be defended if the victim foresees our attack and is aware of the number of adversarial agents. Therefore, in this part, we try multiple potential defense methods to explore the performance of our attack facing different countermeasures. 

\textbf{Simple retrain}.In this experiment, we retrain the policy of victim agents in map D and fix the policy of 3 adversarial agents. After retraining, we observe the performance of victim agents when executing tasks under attack from adversarial agents and executing tasks without attack.

\begin{table}
    \centering
    \begin{tabular}{c|c|c} \hline
        \textbf{Wining rate} & \textbf{Before retraining} & \textbf{After retraining}\\  \hline 
        \textbf{Under attack} & 0.0 & 0.8\\  \hline
        \textbf{Without attack} & 0.95 & 0.35\\ \hline
    \end{tabular}
    \caption{Comparison of capability between victim agents before and after retrain on map "6h\_vs\_8z"}
    \label{tab:1}
\end{table}

\begin{table}
    \centering
    \begin{tabular}{c|c|c} \hline
         & Under attack & No attack \\\hline
       CAMP\_H  & 0.44 & 0.82 \\\hline
       CAMP\_S  & 0.06 & 1.0 \\\hline
       PATROL  & 0.82 & 0.98 \\\hline
       PATROL\_R  & 0.38 & 0.98 \\\hline
    \end{tabular}
    \caption{The performance of our attacks facing against existing countermeasures. CAMP\_H is using CAMP as the countermeasure in situation "intersection\_S" and CAMP\_S is using CAMP as the countermeasure in Starcraft \uppercase\expandafter{\romannumeral 2} with the map "1m". PATROL is using PATROL as the countermeasure in Starcraft \uppercase\expandafter{\romannumeral 2} with the map "1m" and PATROL\_R is re-attack after retrain adversaries under modified adversarial density. }
    \label{tab:3}
\end{table}

As shown in 
Appendix \ref{AP6} and 
Table \ref{tab:1}, victim agents do not converge during retrain in map D, and the retrain also influences significantly in normal task without attack. 

\textbf{Existing reachable defense method}. Recently, researchers have proposed several defense and detection methods for DRL, which can mainly be categorized into adversarial training based defense method ~\cite{behzadan2017whatever, mandlekar2017adversarially, pattanaik2017robust, guo2023patrol}, noise based defense method ~\cite{behzadan2018mitigation, wang2025campodysseyprovablyrobust} and detection of adversarial examples ~\cite{havens2018online, lin2017detecting}. On the one hand, as discussed above, a simple adversarial retrain cannot well-defend our attacks. Inspired by PATROL ~\cite{guo2023patrol}, only through a game-theoretic reformulation of the optimization problem — seeking an optimal balance between attack resilience and the model's primary task performance — can an effective defense against our attack be achieved. However, PATROL ~\cite{guo2023patrol} is designed by using Stackelberg game model as theory fundament relying on two-agent zero-sum competitive environment, which inherently restricts its applicability to defending against our attacks. On the other hand, noise based defense methods and existing detection methods are only effective on the environment-manipulation based attack against DRL. Therefore, none of the existing work can detect or defend our proposed attacking method. Still, we evaluate our attacks facing some of the typical existing countermeasures including CAMP ~\cite{wang2025campodysseyprovablyrobust} (environment-manipulation based defense) and PATROL ~\cite{guo2023patrol} (adversarial retrain). As we can observe from Table \ref{tab:3}, the environment-manipulation based defense such as CAMP ~\cite{wang2025campodysseyprovablyrobust} cannot affect the attacks from adversarial policy training methods. While PATROL can effectively defend our attack when there is only one victim agent and one adversarial agent in a zero-sum competitive environment and posit that the defender possesses prior knowledge of the attacker's agent. However, the aforementioned assumptions are operationally untenable, and model robustness collapses when subjected to re-trained attackers under modified adversarial density.

As discussed above, both \textbf{simple retrain} and \textbf{existing countermeasures} fail to defend our attacks effectively and practically. Therefore, we carefully conclude that our proposed method is very hard to defend against with existing techniques. 

\section{Related Work}\label{section_related_work}

Existing research on the security of DRL can be categorized into two types: environment manipulation-based methods and adversarial policy learning-based methods. In the following sections, we review representative works in each category and highlight their distinctions from our proposed method.
\subsection{Environment manipulation-based methods}

In the field of secure research on deep learning, many studies have demonstrated that neural networks are highly sensitive to adversarial perturbations ~\cite{carlini2017towards, goodfellow2014explaining, gu2017badnets, papernot2016limitations}. Attackers can exploit adversarial training by adding noise to the neural network's input to force misclassification. Researchers in the domain of deep reinforcement learning have applied this discovery to secure research by adding noise to an agent's observations, thereby preventing the agent from making optimal decisions.

In existing work, Huang \textit{et al.} ~\cite{huang2017adversarial} demonstrated that adversarial learning can easily be used to propagate noise into policy networks, causing the agent to lose the game. Subsequent studies ~\cite{kos2017delving, lin2017tactics, russo2019optimal} improved upon these methods, enhancing the efficiency of such attacks. In recent research, researchers have extended adversarial attacks to cooperative multi-agent algorithms, attempting to disrupt multi-agent collaboration by using adversarial training or custodial attacks to interfere with, manipulate, or alter the observations, reward signals, or specific actions of individual agents ~\cite{lin2020robustness, zan2023adversarial, wu2023reward, liu2023efficient}. However, both of these methods assume that the attacker has the ability to monitor and overwrite the observation, reward, or reward signal, which is highly impractical due to the significant overhead involved.

\subsection{Adversarial policy learning-based methods}
Gleave \textit{et al.} ~\cite{gleave2019adversarial} were the first to introduce adversarial policy. Distinct from model manipulation attacks, adversarial policy attacks do not necessitate access to victim observation, action or reward. Instead, they introduce an adversarial agent to deceive victim agents with well-designed actions, causing victim to take counterintuitive actions and ultimately fail to achieve their goals. Wu \textit{et al.} ~\cite{wu2021adversarial} induced larger deviations in victim actions by perturbing the most sensitive feature in victim observations, and Guo \textit{et al.} ~\cite{guo2023patrol} extended adversarial policies to general-sum games. However, these researches focus on attacking single RL agent in competitive environment. On the one hand, these studies use -rvictim as reward without further design. Using this simple reward in attacking c-MARL will cause most of feedback at the beginning of training being negative reward, making policy hard to converge. On the other hand, none of these studies have considered adversarial policies in c-MARL settings. Simin \textit{et al.} ~\cite{li2023attacking} introduced Adversarial Minority Influence, a black-box policy-based attack for c-MARL, driving minorities (attackers) unilaterally sway majorities (victims) to adopt its own targeted belief. However, this research based on the authority to access at least one agent of c-MARL agent set. Under this assumption, attacker either find a way to spy in the agent group of c-MARL or hack into the at least one agent in c-MARL. In most cases, such expectation cannot be guaranteed. Besides, this research fails to prove the reason why victims fail is the influence of adversarial agent, or just losing an ally breaking the cooperation of MARL. In our experiments, we discover that even make one agent of c-MARL agent set act randomly could severely increase the failure rate.
\section{Discussion}\label{section_discustion}

\subsection{Deep learning and rule-based method}

In this study, we investigated the attacks and effects of adversarial policy training on deep reinforcement learning (DRL) across various environments, extending it to a general attacking method for DRL. In the future, we plan to continue exploring the application of adversarial policy training in attacking methods beyond DRL. Specifically, we aim to apply adversarial policy training against deep neural networks (e.g., RNN, LSTM) or rule-based methods in sequential decision-making tasks. To achieve this goal, several challenges must be addressed. First, in this work, both the victim and attacker utilize algorithms such as DQN or QMIX, which can be modeled as Dec-POMDP. However, in sequential decision-making tasks beyond DRL, the victim no longer employs reinforcement learning algorithms, or even not an agent. Additionally, scenarios utilizing DNNs, large models, or rule-based methods are significantly more diverse compared to those employing DRL. Therefore, under this circumstance, migrating our attacking method might require significant modification or even a completely new design.  Second, the design of the objective function of adversarial agents in this work is based on estimates part of the victim's action-value functions. However, in non-DRL methods, such value or action-value functions may not exist, thus, extending our attacking method might require the redesign of the objective function for different task scenarios. Third, in the environments of RL, both the attacker and victim operate in real-time under the same state. Yet, many sequential decision-making tasks that do not employ DRL are non-real-time, meaning that the roles represented by the attacker and victim may not operate simultaneously. Consequently, the attacker cannot design an immediate reward function based on the victim's current state and task situation, rendering the rule-based reward model used in this work ineffective and potentially requiring major revisions to the estimation-based reward model.

\subsection{Transferability}

Recent research ~\cite{huang2017adversarial} indicate strong transferability of adversarial attacks within reinforcement learning environments. Specifically, an attack or interference targeting one policy network can be easily transferred to another distinct policy network under same reinforcement learning task. For instance, an effective attack against a reinforcement learning agent utilizing LSTM as its policy network can be quickly adapted and applied to disrupt a reinforcement learning agent employing MLP as its policy network. In future work, we aim to investigate the transferability of our adversarial policy attacking approach. We will evaluate whether an adversarial policy trained on victim agents using a specific algorithm under the same task environment can effectively attack victim agents using an alternative algorithm.

\section{Conclusion}\label{section_conclusion}

In this paper, we propose an adversarial attack method against DRL in multi-party open systems based on adversarial policy training with multiple neutral agents. Different from existing studies, we do not manipulate either environments or victim agents and we do not require direct interactions with victim agents as either competitive or cooperative standings. Besides, we design our method to cover as many scenarios as possible in multi-party open systems. 
Technically, we redesign the reward function by exploring different failure paths of each scenario to address the challenge of reward shaping. Furthermore, we propose an estimation-based reward model, which estimates the reward for adversarial agents with partial observations using the LSTM network without the requirement of the global state in each step. 
The evaluation demonstrates that our proposed method performs well on attacking victim agents under varying scenarios in multi-party open systems. Our estimation-based reward model is also proved to be more effective compared with baselines. With above discoveries and discussions, we safely conclude that our proposed attack method is much more general and can deploy effective attacks against DRL in various open environments.

\bibliographystyle{ACM-Reference-Format}
\bibliography{references}

\appendix
\section{Further discussion of assumptions}
\label{AP7}

 Sharing the similar point with ~\cite{gleave2019adversarial, wu2021adversarial, guo2021adversarial}, in this work, we assume only adversarial party agents adapt their policy in a multi-party open system immediately. With this assumption, we take a real-world scenario of a multi-party open system as an example, where a group of vehicles controlled by agents perform on-ramp merging tasks on the highway and other vehicles controlled by agents or human drivers with their own tasks, going straight, lane-changing, or overtaking, etc. A group of RL agents requires millions of episodes of training and multiple circumstance evaluations to ensure its ability and safety ~\cite{OpenAI2019}, which takes years of time to retrain the model, collect data, and design experiments. Therefore, participants cannot afford to retrain the algorithm and update it on every intelligent driving system in a short period of time.
 
Besides, it should be noted that this work does not assume that we can manipulate the environment or any agents of the victim party or those not belonging to the attacker. Instead, we assume that attacks occur only in the open environment that allow deploying third party agents at any time without directly participating in task of victim agents. We believe the replace of this assumption is crucial and could make an adversarial attack more practical. To illustrate this argument, we again take for example the aforementioned autonomous driving task. In this circumstance, manipulation of any vehicles that not belong to attackers means break into the intelligent driving system, alters the code related to the autonomous driving, and thus influences the environment that the agents interact with. This is not practical, as it would require thousands of hours of effort from professional hackers and does not guarantee the successful identification of software vulnerabilities or the acquisition of control. In most reinforcement learning applications, the environments are open and there are no restrictions on the deployment of other agents. However, maliciously causing damage to devices owned by others is generally prohibited by rules and might cause significant losses for attacker. Therefore, we assume that attacks occur in an open environment, but adversarial agents are not allowed to take any actions that would directly participate in the tasks of victim agents.

It should also be noted that most research on adversarial policies relies on a hidden assumption to calculate the reward signal: during every steps in an episode, adversarial agents and victim agents share global observations ~\cite{wu2021adversarial, guo2021adversarial, li2023attacking}. This means that adversarial agents can observe all relevant agents (including victim agents, potential target agents, and third-party agents) in the entire environment, as well as the task completion status of victim agents. This assumption is natural in two-player competitive environments because the task scenarios are often narrow and involve only adversarial agents and victim agents. However, in non-competitive environments, the scenarios are often more diverse and complex, with many third-party agents and open environments. In such cases, no agent can quickly obtain global information. Therefore, we can only assume that all agents share the global state transition function, but agents from different parties cannot share information, and no agent can rapidly obtain the global state through observation.

\section{Example figure of possible failure paths}
\label{AP5}

\begin{figure*}[ht!]
\centering
\subfloat[Failure path 1]{\includegraphics[width=0.32\linewidth]{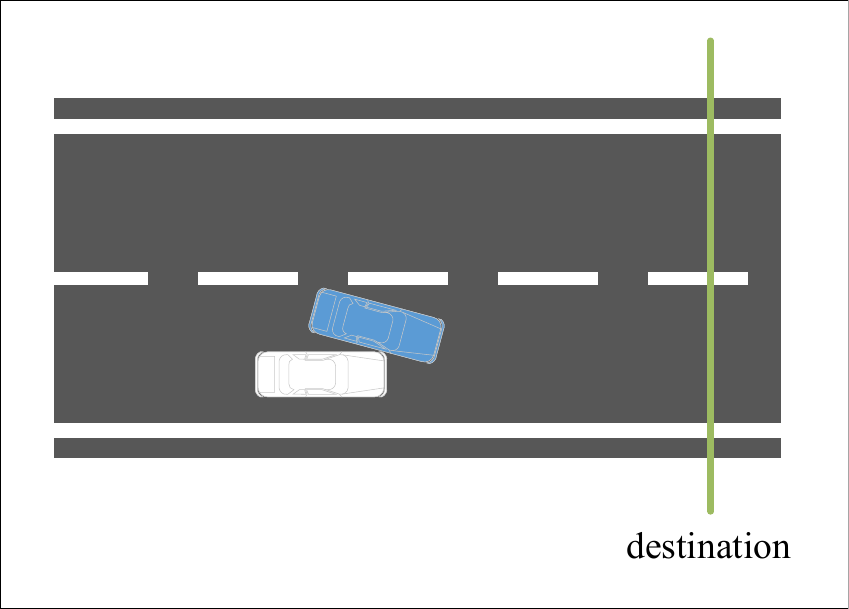}\label{figS5_1a}}
\subfloat[Failure path 2]{\includegraphics[width=0.32\linewidth]{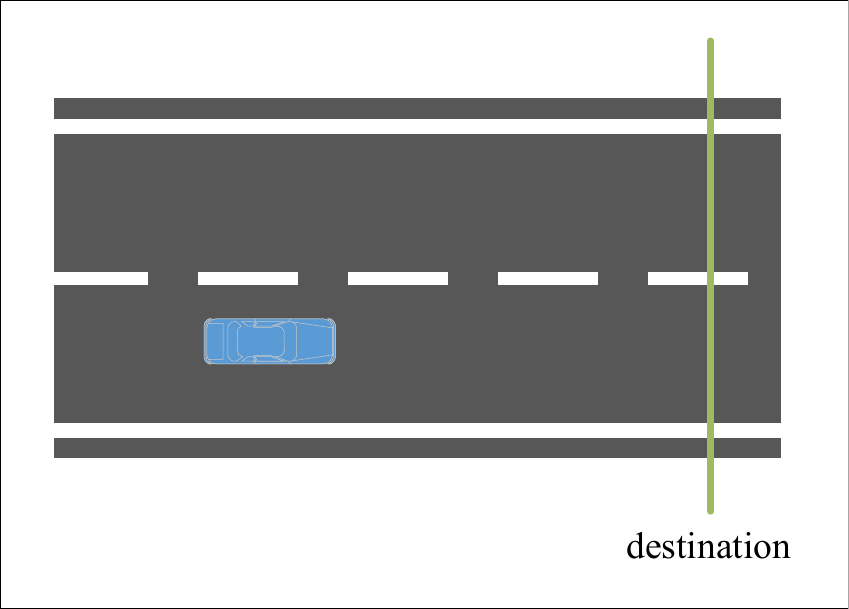}\label{figS5_1b}}
\subfloat[Failure path 3]{\includegraphics[width=0.32\linewidth]{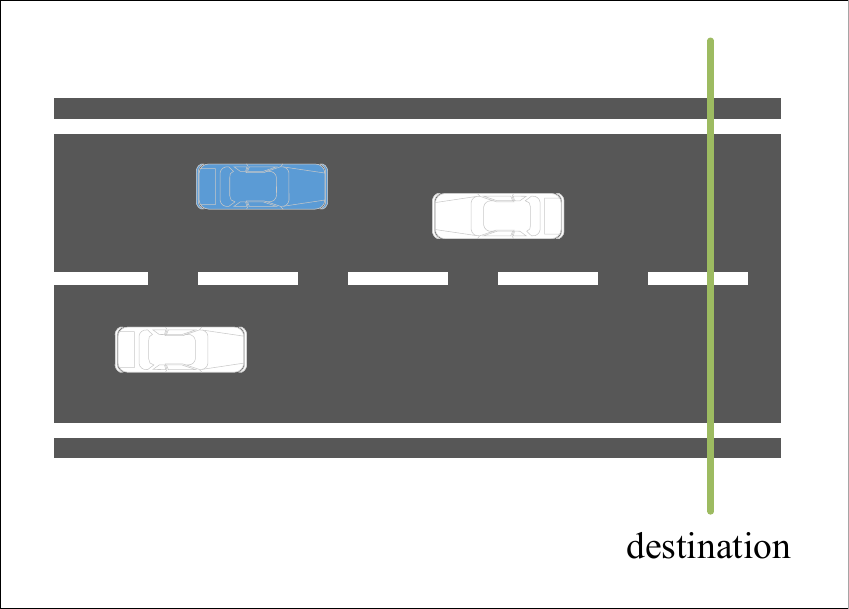}\label{figS5_1c}}
\caption{Possible failure paths of autonomous driving task: collision occured, unreach destination before time limitation, and disobey the traffic rule (drive against the traffic flow).}
\label{fig:S5_1}
\end{figure*}

The example of possible failure paths are shown as Figure \ref{fig:S5_1}

\setcounter{proposition}{0}
\setcounter{equation}{0}
\section{Proof of Proposition 1}
\label{AP1}

\begin{proposition}
    In a multi-party open system, if all agents follow fixed policies except agents of one specific party, the state transition of the environment system will depend only upon the joint policy of agents belonged to this specific party rather than the joint policy of all agents in the system.
\end{proposition}

\begin{proof}
    We divide the open environment POMDP into three parts: adversaries, victims, and other third-party agents, separately denoted as \(\alpha\), \(v\) and \(\tau\), each of which contains several independent agents and takes joint action at each step. We assume that agents in \(v\) and \(\tau\) follow fixed policies, and agents in \(\alpha\) can update policies adversarially. At global state \(S_t\), the probability of taking the joint actions \((A_t^\alpha, A_t^v, A_t^\tau)\) and transiting to \(S_{t+1}\) is: 

    \begin{equation}
        \begin{aligned}
            &\ \ \ \ \ P(S_{t+1}, A_t^\alpha, A_t^v, A_t^\tau | S_t) \\ &= P(S_{t+1} | A_t^\alpha, A_t^v, A_t^\tau, S_t)P(A_t^\alpha, A_t^v, A_t^\tau | S_t) \\
            &=P(S_{t+1} | A_t^\alpha, A_t^v, A_t^\tau, S_t)P(A_t^\alpha | A_t^v, A_t^\tau , S_t)P(A_t^v, A_t^\tau | S_t) \\
            &= P(S_{t+1} | A_t^\alpha, A_t^v, A_t^\tau, S_t)\pi^\alpha(A_t^\alpha | S_t)\pi^v(A_t^v|S_t)\pi^\tau(A_t^\tau|S_t) \\
            &= c \cdot P(S_{t+1} | A_t^\alpha, A_t^v, A_t^\tau, S_t)\pi^\alpha(A_t^\alpha | S_t),
        \end{aligned}
        \label{A1_1}
    \end{equation}

    where \(c = \pi^v(A_t^v|S_t)\pi^\tau(A_t^\tau|S_t)\). Given that at a time step \(t\), the joint action of adversaries \(A_t^\alpha\) depends only upon the current state \(S_t\), we have \(\pi^\alpha(A_t^\alpha | S_t) = P(A_t^\alpha | A_t^v, A_t^\tau , S_t)\). 

    As we can observed from Equation (\ref{A1_1}), during the adversarial training process, the only part that changes agents' policies is \(\alpha\). Therefore, the changes in every agents' value functions and Q-values are determined by the changes of \(\pi^\alpha\). Mathematiclly, given a set of trajectories \(\{tr_1, tr_2, \dots, tr_m\}\), the Q-value functions of each agent \(i\) in \(\alpha\) can be denoted as:

    \begin{equation}
            Q_i^{\pi^\alpha} = R_i^\alpha(S,A_i) + \gamma\sum_{S'}P(S'| S, A_i)V_i^\alpha(S').
        \label{A1_2}
    \end{equation}

    In Equation (\ref{A1_2}), 

    \begin{equation}
        \begin{aligned}
            &V_i^\alpha(S') = \sum_{m=1}^M R^\alpha(\tau_m)P(\tau_m;\theta),\\
            &P(\tau;\theta) = P(S_0)\prod_{t=0}^{T-1} P(S_{t+1}, A_t^\alpha, A_t^v, A_t^\tau | S_t)
        \end{aligned}
        \label{A1_3}
    \end{equation}

    Similar to Equation (\ref{A1_1}), in Equation (\ref{A1_3}), the only part that changes agents' policies is \(\alpha\). With plugging Equation (\ref{A1_3}) into Equation (\ref{A1_2}), we carefully conclude that only the changes in joint policy \(\pi^\alpha\) of adversaries, rather than the joint policy of all agents in system,  determine the change in value functions and Q-value functions of each adversary. 

\end{proof}

\section{Proof of Proposition 2}
\label{AP2}

\begin{proposition}
    \label{prop_2}
    The long-horizon expected return sharing the same weighted changes with the short-term reward function: 
\begin{equation}
  Q^{\pi^\alpha}_i(S,A_i^\alpha)=W \times (R_i(S,A_i)+\gamma \sum _{S'}P(S'|S,A_i^\alpha)V_i(S'))^\top . 
  \label{A2_1}
\end{equation}
\end{proposition}

\begin{proof}
    As mentioned in Section \ref{section_tech_overview}, with our reshaped reward functions, the independent Q-values of each adversarial agent can be written as:

\begin{equation}
  Q^{\pi^\alpha}_i(S,A_i^\alpha)=W \times R_i(S,A_i)^\top+\gamma \sum _{S'}P(S'|S,A_i^\alpha)V_i^{\pi^\alpha}(S') . 
  \label{A2_2}
\end{equation}

In Equation (\ref{A2_1}) and (\ref{A2_2}):

\begin{equation}
    \begin{aligned}
      &V_i(s') = \sum_a \pi(a|s')\sum_{s''}p(s''|s',a)[R_i(s', a)+\gamma V_i(s'')], \\
      &V_i^{\pi^\alpha}(s') = \sum_a \pi(a|s')\sum_{s''}p(s''|s',a)[R_i^{\pi^\alpha}(s', a)+\gamma V_i^{\pi^\alpha}(s'')].
    \end{aligned}
    \label{A2_3}
\end{equation}

By comparing Equation (\ref{A2_1}) with Equation (\ref{A2_2}), we demonstrate that proving Proposition \ref{prop_2} reduces to verifying the equality \(V_i^{\pi^\alpha}(s') = W \times V_i(s')^\top\). Subsequently, we employ mathematical induction to prove this equality. 

With a finite and complete set of trajectories \(\{tr_{n-(n-1)}, tr_{n-(n-2)}, \dots , tr_{n-m}, \dots, tr_{n-1}, tr_{n-0}\}\), we proceed by mathematical induction on \((m)\)  for \(m \in [0, n-1]\).

\textbf{Base Case}. For \((m=0)\), \(tr_{n-m}\) is the last trajectory of the set, from which we have \(V_i^{\pi^\alpha}(s_{n+1}) = W \times V_i(s_{n+1})^\top = 0\). Therefore, we have the following mathematical derivations:

\begin{equation}
    \begin{aligned}
      V_i^{\pi^\alpha}(s_{n-m}) &= V_i^{\pi^\alpha}(s_{n}) = \sum_a \pi(a|s_n)\sum_{s_{n+1}}p(s_{n+1}|s_{n},a)R_i^{\pi^\alpha}(s_{n}, a) \\
      &=\sum_a \pi(a|s_n)\sum_{s_{n+1}}p(s_{n+1}|s_{n},a)W\times R_i(s_{n}, a)^\top \\
      &= W \times V_i(s_{n})^\top \\
      &= W \times V_i(s_{n-m})^\top.
    \end{aligned}
\end{equation}

\textbf{Inductive Hypothesis}. Assume the equality \(V_i^{\pi^\alpha}(s') = W \times V_i(s')^\top\) holds for \((m=k)\), which is \(V_i^{\pi^\alpha}(s_{n-k}) = W \times V_i(s_{n-k})^\top\).

\textbf{Inductive Step}. For \((m=k+1)\), we have:

\begin{equation}
    \begin{aligned}
      &\ \ \ \ \ \ V_i^{\pi^\alpha}(s_{n-m}) = V_i^{\pi^\alpha}(s_{n-(k+1)}) \\ &= \sum_a \pi(a|s_{n-(k+1)})\sum_{s_{n-k}}p(s_{n-k}|s_{n-(k+1)},a) \times \\ 
      &\ \ \ \ \ [R_i^{\pi^\alpha}(s_{n-(k+1)}, a) + \gamma V_i^{\pi^\alpha}(s_{n-k})] \\
      &=\sum_a \pi(a|s_{n-(k+1)})\sum_{s_{n-k}}p(s_{n-k}|s_{n-(k+1)},a) \times \\
      &\ \ \ \ \ [W\times R_i(s_{n-(k+1)}, a)^\top + \gamma W \times V_i(s_{n-k})^\top] \\
      &= W \times V_i(s_{n-(k+1)})^\top \\
      &= W \times V_i(s_{n-m})^\top.
    \end{aligned}
\end{equation}

\textbf{Conclusion}. By induction, the equality \(V_i^{\pi^\alpha}(s') = W \times V_i(s')^\top\) is valid for a finite and complete set of trajectories \(\{tr_{n-(n-1)}, tr_{n-(n-2)}, \dots , tr_{n-m}, \dots, tr_{n-1}, tr_{n-0}\}\) where \(m \in [0, n-1]\). 

\end{proof}
\section{Neutral agent-based adversarial policy learning algorithm}
\label{AP3}

Neutral agent-based adversarial policy learning algorithm is displayed as \ref{alg}.

\begin{algorithm}
\renewcommand{\algorithmicrequire}{\textbf{Input:}}
\renewcommand{\algorithmicensure}{\textbf{Output:}}
\caption{Neutral agent-based adversarial policy learning algorithm}
\label{alg}
\begin{algorithmic}[1]
    \REQUIRE the Deep Q Networks of adversarial agents' policies \(\pi_i\) parameterized by \(\theta_i\) where \(i\in[1,N]\), the mixing network of QMIX \(QT\) parameterized by \(\theta_q\), the LSTM Network \(M\) of reward model parameterized by \(\theta_m\) if using estimation-based reward model, a set of well trained victim agents \(V_i\) where \(i\in [1,H]\), a state transition function F.
    \ENSURE A set of well-trained adversarial policy network \(\pi_i\) and reward estimator LSTM network \(M\).
    \STATE \textbf{Initialization:}\(\theta_i\), \(\theta_q\), \(\theta_l\), hidden state \(h\) of LSTM
    \FOR{\(k=0,1,2...K\) do}
        \STATE Reset environment global state to \(S_0\)
        \FOR{\(t=0,1,2...T\) do}
            \FOR{\(i=0,1,2...N\) do}
                \STATE Adversarial agent i get observation \(o_i\)  and available actions \(a_{i}^{av}\) from \(S_t\)
                \STATE Adversarial agent i choose action: \\ \(a_i=argmax_a \pi_i(o_i, a_{i}^{av})\)
            \ENDFOR
            \STATE Each agents in environment take joint action \(A_t\)
            \STATE Update global state S by state function F:
            \\
            \(S_{t+1}=F(S_t, A_t)\)
            \STATE Calculate reward \(r_t\) of adversarial agents by reward estimator LSTM network \(M\).
        \ENDFOR
        \STATE Collect a set of trajectories \(D^k\) where \(D^k_t=(o_t, A_t^{adv},r_t)\)
        \IF{reward model is estimation-based}
            \STATE Update \(\theta_m\) by loss function (\ref{16})
        \ENDIF
        \STATE Choose a trajectory \(D^k\) to update adversarial policy
        \STATE Compute independent Q value for each adversarial agent \(i\) at each step \(t\): \(Q_i=\pi_i(a_i^t, o_i^t)\)
        \STATE Compute total Q \(Q^{tot}=QT(q_1,q_2,...q_N)\)
        \STATE Update \(\theta_i\), \(\theta_q\) by loss function (\ref{s4_10})
    \ENDFOR
\end{algorithmic}
\end{algorithm}

\section{Detailed introduction of evaluation platform}
\label{AP4}

\textbf{Starcraft \uppercase\expandafter{\romannumeral 2} and SMAC}. Starcraft \uppercase\expandafter{\romannumeral 2} is a real-time strategy game developed by Blizzard Entertainment and released on July 27, 2010. It involves one or more players competing against each other or built-in game AI by gathering resources, constructing buildings, and assembling armies to defeat opponents. The decision-making process in StarCraft \uppercase\expandafter{\romannumeral 2} can be divided into two main categories: macro decisions and micro decisions. Macro decisions involve high-level strategic considerations, such as economic and resource management, while micro decisions involve fine-grained control operations over individual units.

To better demonstrate the evolving capabilities of reinforcement learning agents, their evaluation often places greater emphasis on micro decisions. In the context of StarCraft \uppercase\expandafter{\romannumeral 2}, micro decisions have a very high skill ceiling, requiring both amateur and professional players to repeatedly practice and improve this ability. When testing multi-agent reinforcement learning (MARL), each unit is controlled by an independent agent, which must be trained to complete challenging combat scenarios based on local observations. These agents aim to maximize damage dealt to enemy units while minimizing self-inflicted damage, collaborating with each other to defeat enemies.

SMAC consists of a set of micro-scenarios in StarCraft \uppercase\expandafter{\romannumeral 2}, designed to evaluate the ability of reinforcement learning algorithms to learn how to solve complex tasks. In these well-designed scenarios, agents must learn micro-level operations to defeat enemies. Each scenario involves a combat between two opposing militaries. The terrain, initial positions, quantities, and unit types of each military vary depending on the specific scenario.

The maps in SMAC typically involve two opposing militaries. In our experiments, the victim agents control one of the militaries, while the other is controlled by the game's built-in programs. Based on the assumptions of our work, we added third-party agents into the game environment. These third-party agents do not directly cause harm to the military controlled by the victim agents. The objective of victim agents is to defeat the military controlled by the game's built-in programs within limited time. During the training of the victim agents, the behaviors of all third-party units are completely random. Once the victim agents are well-trained, we train a few third-party agents as adversarial agents with fixed victim policy. In our experiment, we follow the metric commonly used for evaluating reinforcement learning, measuring the winning rate and average reward of the adversarial agents at each iteration.

\textbf{Highway-Env}.Highway-Env is an open-source Python simulation environment specifically designed to facilitate research in decision-making for autonomous vehicles, with a focus on behavioral planning and motion planning using Reinforcement Learning (RL). Developed to provide a lightweight, modular, and highly configurable platform, it abstracts low-level vehicle dynamics through simplified kinematic models to prioritize learning high-level tactical maneuvers. The environment features diverse, configurable driving scenarios—including multi-lane highway navigation, roundabout negotiation, goal-oriented parking, unsignalized intersection crossing, and racetrack driving—that model critical interactions like lane changes, overtaking, merging, and congestion handling. Its core strength lies in seamless compatibility with standard RL frameworks (via OpenAI Gym/Gymnasium APIs), extensive configurability of road networks, traffic parameters, reward functions, observation spaces (e.g., state vectors, occupancy grids), and action spaces (discrete or continuous). While its integrated PyGame-based visualization supports debugging and its low computational footprint enables rapid prototyping on standard hardware, Highway-Env deliberately sacrifices high-fidelity physics and realistic sensor simulation (e.g., cameras, LiDAR) to concentrate research efforts on strategic decision-making. Consequently, it serves as an accessible and efficient tool for developing, benchmarking, and evaluating autonomous driving algorithms, particularly within RL research and educational contexts, despite simplifications in background traffic behavior and vehicle kinematics.

\section{Supplementary experiment material}
\label{AP6}

\begin{figure*}
    \centering
    \includegraphics[width=1\linewidth]{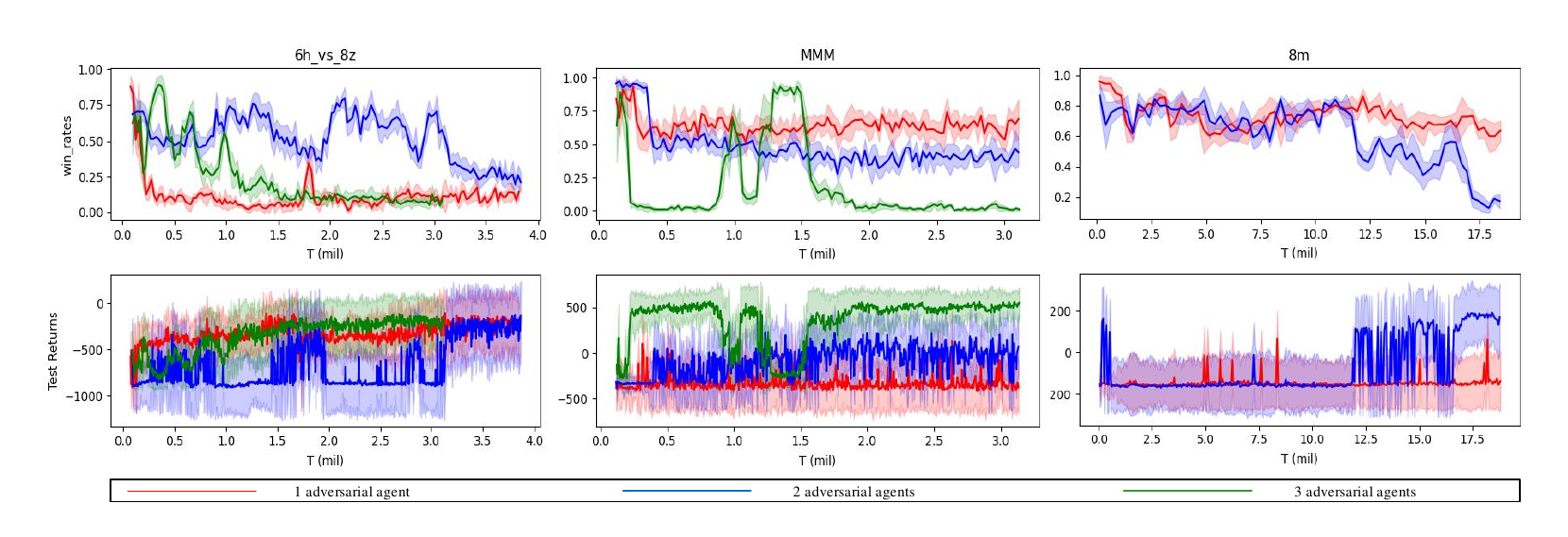}
    \caption{Comparison of wining rates and rewards trend during training across deploying from 1 to 3 adversarial agents in Starcraft \uppercase\expandafter{\romannumeral 2} maps}
    \label{fig:5}
\end{figure*}

\begin{figure}
    \centering
    \includegraphics[width=1\linewidth]{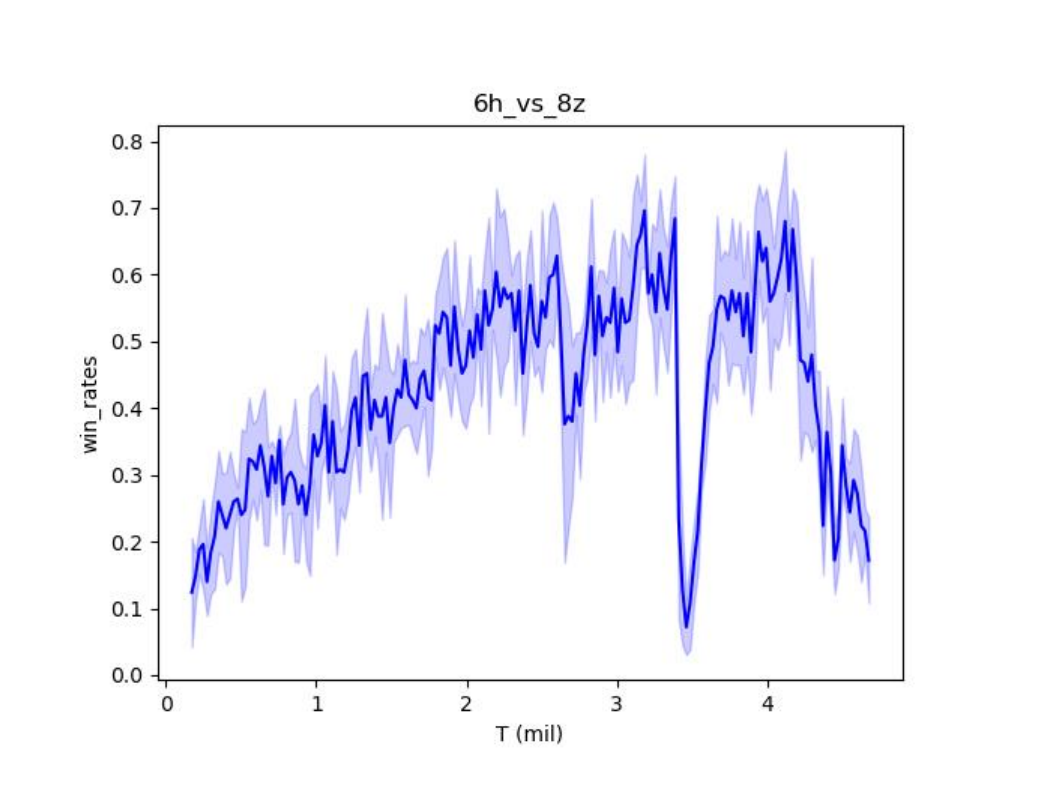}
    \caption{Winning rate status when retrain victim agents on map "6h\_vs\_8z" with 3 adversarial agents}
    \label{fig:7}
\end{figure}










\end{document}